\title{OD Transactions Paper}
\documentclass[twoside, 11pt]{article}

\usepackage{jmlr2e}

\firstpageno{1}

\usepackage{url}

\usepackage{cite}
\usepackage{graphicx}

\usepackage{amsmath}
\usepackage{amssymb}
\usepackage{yfonts}
\usepackage{bm}

\usepackage[T1]{fontenc}
\usepackage[libertine,cmintegrals,cmbraces,vvarbb]{newtxmath}
\usepackage[scaled=0.95]{inconsolata}
\usepackage{fix-cm}

\usepackage{color}
\usepackage{soul}
\usepackage{enumitem}
\usepackage[us]{datetime}
\usepackage{array, booktabs, makecell}  

\usepackage[table]{xcolor} 

\newcommand{\argmin}{\operatornamewithlimits{arg\ min}}

\newcommand{\beq} {\begin{equation}}
\newcommand{\eeq} {\end{equation}}
\newcommand{\beqs} {\begin{equation*}}
\newcommand{\eeqs} {\end{equation*}}
\newcommand{\f}[2]{\frac{#1}{#2}}

\newcommand{\ck}{\ensuremath{\mathfrak{K}}}
\newcommand{\sB}{\ensuremath{\mathcal{B}}}
\newcommand{\sE}{\ensuremath{\mathcal{E}}}
\newcommand{\sF}{\ensuremath{\mathcal{F}}}

\newcommand{\sH}{\ensuremath{\mathcal{H}}}
\newcommand{\sJ}{\ensuremath{\mathcal{J}}}
\newcommand{\tsJ}{\ensuremath{{\tilde{\mathcal{J}}}}}
\newcommand{\sL}{\ensuremath{\mathcal{L}}}
\newcommand{\sP}{\ensuremath{\mathcal{P}}}
\newcommand{\sR}{\ensuremath{\mathcal{R}}}
\newcommand{\sS}{\ensuremath{\mathcal{S}}}
\newcommand{\sT}{\ensuremath{\mathcal{T}}}
\newcommand{\sX}{\ensuremath{\mathcal{X}}}
\newcommand{\sY}{\ensuremath{\mathcal{Y}}}
\newcommand{\sZ}{\ensuremath{\mathcal{Z}}}
\newcommand{\bF}{\ensuremath{\mathfrak{F}}}
\newcommand{\bL}{\ensuremath{\mathfrak{L}}}
\newcommand{\Mm}{\ensuremath{\mathfrak{m}}}     
\newcommand{\E}{\mathbb{E}}
\newcommand{\Rbb}{\mathbb{R}}

\newcommand{\n}[1]{\|#1\|}
\newcommand{\nbig}[1]{\ensuremath{\bigg \| #1 \bigg\|}}
\newcommand{\set}[1]{\{#1\}}
\newcommand{\ip}[1]{\langle #1 \rangle}
\newcommand{\onev}{\ensuremath{\textbf{1}}}
\newcommand{\BgB}[1]{\ensuremath{\big(#1\big)}}
\newcommand{\BggB}[1]{\ensuremath{\bigg(#1\bigg)}}

\DeclareMathOperator{\sgn}{sgn}

\newcommand{\qed}{\nobreak \ifvmode \relax \else
    \ifdim\lastskip<1.5em \hskip-\lastskip
    \hskip1.5em plus0em minus0.5em \fi \nobreak
    \vrule height0.75em width0.5em depth0.25em\fi}

\begin{document}
        
\title{Kernel Embedding Approaches to Orbit Determination of Spacecraft Clusters}
    
\author{\name Srinagesh Sharma \email srinag@umich.edu \\
       \addr Electrical Engineering and Computer Science\\
       University of Michigan
       \AND
       \name James W. Cutler \email jwcutler@umich.edu \\
       \addr Aerospace Engineering\\
       University of Michigan}

    
    \maketitle
    
    \begin{abstract}
        This paper presents a novel formulation and solution of orbit determination over finite time horizons as a learning problem. We present an approach to orbit determination under very broad conditions that are satisfied for n-body problems. These weak conditions allow us to perform orbit determination with noisy and highly non-linear observations such as those presented by range-rate only (Doppler only) observations. We show that domain generalization and distribution regression techniques can learn to estimate orbits of a group of satellites and identify individual satellites especially with prior understanding of correlations between orbits and provide asymptotic convergence conditions. The approach presented requires only visibility and observability of the underlying state from observations and is particularly useful for autonomous spacecraft operations using low-cost ground stations or sensors. We validate the orbit determination approach using observations of two spacecraft (GRIFEX and MCubed-2) along with synthetic datasets of multiple spacecraft deployments and lunar orbits. We also provide a comparison with the standard techniques (EKF) under highly noisy conditions.
    \end{abstract}

    \section{Introduction}

    The last decade has seen a rapid growth in space scientific endeavors and exploration. Increased access to space has resulted in an exponential growth in the number and types of space objects \citep{richardson2015small, berthoud2016set} . Due to mission requirements, communication, debris avoidance, exploration etc., there is an inherent need to track and catalogue space objects. Information for tracking is traditionally cooperative through transponders or GPS-based system, or uncooperatively through radars and telescopes. Orbit state estimation and estimate refinement from these different types of observations is computed through initial orbit determination approaches and Kalman filters \citep{vetter2007pfifty, wright2013odtk}.

    Low-cost access to space has also resulted in large numbers of high-risk, low-cost spacecraft systems such as nanosatellites and CubeSats, which can be deployed near simultaneously in large numbers \citep{bandyopadhyay2016review}. These deployments can be over widely varying orbit ranges from low Earth to deep space \citep{schoolcraft2017marco}. Requirements for successful and efficient mission operations for such growing numbers of spacecrafts have lead to the development of ground station networks with widely varying communications and costs \citep{cutler2006framework, minelli2012mobile, cheung2015next}. 
The increasing number of clustered spacecraft deployments have also lead to an increased nead in autonomy within these networks where scheduling, operations, and tracking are performed autonomously \citep{cutler2006framework, cheung2015next, leveque2007global, colton2016supporting}. 
    
    Achieving mission autonomy for tracking requires consistent access to general orbit determination techniques that can be applied over very broad sets of scenarios. From an operational perspective, it is of great importance in the development of autonomy in ground station network operations if such networks could be used to perform orbit determination and identification of CubeSat constellation deployments. In current spacecraft operational scenarios, such autonomous orbit determination is performed through the use of transponders or GPS receivers on spacecraft. However, transponder-based state estimation can only be performed when spacecraft position uncertainty is small enough to overcome link budget constraints and GPS receivers are applicable only for near Earth orbit satellites. This necessarily implies low initial uncertainty in spacecraft position. In addition, a number of Earth orbiting satellites lack direct navigation capabilities due to absence of GPS systems or transponders \citep{martin2016interplanetary}. For true mission autonomy, it is critical to have methods of orbit determination that could be performed without link budget constraints and transponder uncertainty and timing constraints, where visibility of transmissions and observability of the state through such transmissions are the only criteria that need to be met. 
    
    In addition to autonomous ground station network operations, there are two other areas where a general orbit determination technique is of critical importance. One is in tracking of orbital debris and the other in general tracking of celestial objects. Large numbers of short lifecycle deployments have also resulted in slow decaying space debris whose tracking is critical for avoidance and mission survival \citep{klinkrad2010space, rossi2006collision, tommei2007orbit, schildknecht2007optical, lewis2011space}. However, the approaches to perform debris identification and tracking vary vastly from identification of functional spacecraft \citep{tommei2007orbit, farnocchia2010innovative}. Due to large numbers of space objects, there is a need for autonomous spacecraft identification and tracking using varied, statistically independent and spatially distributed sets of observations. Celestial object tracking, such as the tracking of asteroid trajectories, have different characteristics in observations. The observations are available only over short sections of the trajectory, and as a consequence, short arc methods have become popular \citep{milani2004orbit, ansalone2013genetic, vallado1998accurate}.

    The evolution of trajectories around celestial objects can be modeled using dynamical systems theory \citep{vallado2001fundamentals}. These trajectories cannot be described in closed form except in the most simplified of scenarios (Keplerian orbits). Traditionally, the evolution of observations of some characteristic features of these dynamical systems are modeled as evolution of noisy processes connected to an underlying observable variation. The goal of orbit determination is to estimate an observable state vector to facilitate trajectory prediction. 
    
    Traditionally, the orbit determination problem is treated as a non-linear filtering problem. When the observations allow for good initial estimates, it is possible to use successive prediction-correction to estimate state vectors \citep{vallado2001fundamentals, milani2010theory, wright2013odtk, lee2005nonlinear}. The standard technique for precision orbit determination is the extended Kalman filter (EKF) \citep{vallado2001fundamentals, vetter2007pfifty, milani2010theory}. The EKF is a suboptimal approximation of the Kalman filter for non-linear systems, which has been shown to converge asymptotically when the initial state of the system is in the linear region \citep{krener2003convergence}. Batch processing with equivalent Gauss-Newton methods are also used \citep{crassidis2011optimal}. A second popular approach is using Bayesian and particle filtering approaches for orbit determination \citep{lee2005nonlinear}, where a likelihood-conjugate prior distribution assumption is made regarding the filter parameters. In some approaches proposed by \citep{lee2005nonlinear}, kernel methods are also used in particle filtering. However, the dynamical models are still linearized and EKF based. There has been recent interest in developing methods for initial orbit determination using Gaussian mixture models \citep{psiaki2017gaussian, demars2013probabilistic}, where the distribution generated can be used as a prior distribution in developing a Gaussian mixture approximation of the batch least squares approach \citep{psiaki2017gaussian}.
        
    When the observations have significant noise variances or the system is highly non-linear, successful initialization of non-linear filters may not be possible. An example of this is Doppler only orbit determination. To the best of our knowledge, there exists no method to initialize filters through Doppler only observations (\citet{wright2013odtk} concurs). 


    In machine learning research, the areas of domain generalization \citep{baxter2000jair} and distribution regression \citep{poczos2013distribution} have received increasing attention in the recent years. In the domain generalization setting, the learning system is given unlabeled data to be classified, and must do so by learning to generalize from labeled datasets that represent similar yet distinct classification problems. This can be done through a variety of approaches such as adaptive complexity regularization \citep{baxter2000jair, maurer2009transfer, pentina2015multi}, mapping to common feature spaces \citep{pan2011domain, maurer2013sparse, muandet2013domain} and transfer learning through marginal distributions \citep{blanchard2011generalizing}. In distribution regression, the learning system needs to learn a map from a set of distributions to a separable Hilbert space where the access to the distribution is available only through its samples \citep{poczos2013distribution, szabo2016learning}. We apply the techniques in domain generalization and distribution regression for orbit determination and spacecraft identification.

    We present a novel and general approach to orbit determination of spacecraft and spacecraft constellations. It is a batch method which trades off computational complexity for significantly weaker requirements for tracking. The only requirements imposed are regularity of the observation and output spaces, observability over finite time and availability of observations sufficient to guarantee observability. We differ from traditional approaches in two crucial ways. 
    First, we consider algorithmic convergence in finite time periods with increasing density in observations. In the traditional treatment as non-linear filtering, observations are obtained at some rate (rate may be stochastic) as the dynamic system evolves through time and one is interested in convergence of the state estimates as \(T \rightarrow \infty\). If initial conditions for the filter are outside the linear range and observations are noisy, increased density in observations will not necessarily lead to convergence of estimates. However, we consider algorithms operating over compact (closed and bounded) subsets of time and observations, where convergence is guaranteed as number of observations and training data over this compact set is guaranteed without the need for initial conditions. This achieves estimation under a weaker set of assumptions than proposed in non-linear filtering. Second we shall treat the set of observations over the finite time period as i.i.d samples of time and observed features in the topological space (of finite time and the space of features) under consideration. Noisy versions of these samples can be embedded in Hilbert function spaces. When the stochastic behavior of noise and the observations are known, they can be used to generate example orbit observations. These can be used to train learning algorithms that can learn mappings from orbital parameters to observations. When the mapping satisfies observability and some regularity conditions, the observations can be used on the learning system to perform orbit determination. For spacecraft constellations, the learning system will also provide, as a by product, a label associated with spacecraft. The implementation of the learning system is general and can be integrated into existing ground station network architectures and its implementation does not change with the nature of observations collected. In this paper, we present its theoretical underpinnings, overview and results of implementations. We leave discussions of implementation architecture for future work.

    The learning system is based on two recent developments in machine learning by Blanchard et. al \citep{blanchard2011generalizing} and Szabo et. al \citep{szabo2016learning}. Distribution regression \citep{szabo2016learning} is used for estimating orbital parameters. Marginal prediction or transfer learning \citep{blanchard2011generalizing} is used for selecting embeddings for orbit determination with spacecraft constellations (to improve computational requirements).
     
    The contributions of this paper are as follows:
    \begin{enumerate}
        \item We present a novel model and method using techniques recently developed in machine learning to perform orbit determination of spacecraft and spacecraft constellations.
        \item We provide conditions under which such a system can be applied.
        \item We present consistency analysis for the concatenated application of marginal transfer learning and distribution regression.
        \item We present experimental results of orbit determination and classification of the GRIFEX and MCubed-2 spacecraft using such a system.
        \item We compare the performance of this system with existing EKF based orbit determination systems in the presence of noise.
        \item We present a synthetic orbit determination scenario of estimation of the orbit of a lunar spacecraft (a chaotic system) from one observation station with direction of arrival and range measurements.
    \end{enumerate}

    This paper is organized as follows. Section \ref{sec:background} provides definitions and a short overview of the learning techniques used. Section \ref{sec:probsetup} provides the detailed problem setting. We present mathematical modeling and the proposed approach in section \ref{sec:mlod}. The consistency analysis of the learning system for spacecraft constellations, connections to convergence bounds derived for the learning algorithms used are presented in section \ref{learningtheory}. We present a brief overview of the sampling and estimation architecture, experimental and synthetic data results in section \ref{sec:results}. Conclusions are given in \ref{sec:conclusions}.
    
    \section{Background} \label{sec:background}
    
    In this section we introduce our mathematical notation, and some theoretical concepts in probability and set theory that will be useful for analysis of the system.
    
    \paragraph*{Notation}
    Upper-case symbols are used to denote random variables or sets. Scalars or vectors are differentiated by context. Lower-case symbols are used to denote either instances of the random variable or known/observed constants. Script letters such as \(\sX, \sY\) etc. are used to denote a measurable space with \(\bF_{\sX}, \bF_{\sY}\) etc., denoting the corresponding Borel \(\sigma\)-algebra and \(P_X, P_Y\) denoting the probability distributions, respectively. The symbol \(\widehat{y}\) is used to denote an estimate of the corresponding true value \(y\).  Subscript \(T\) refers to the test system.
    
    \paragraph{Set distance} We shall define the distance between two measurable sets \(A\) and \(B\) to be \(d_S(A,B) = {\mathfrak{m}}(A \Delta B)\), where \(\mathfrak{m}\) is the Lebesgue measure and \(A \Delta B = A \setminus B \cup B \setminus A\).
    
    \paragraph{Probability kernels and the Prokhorov metric}
    The space of probability distributions on a compact metric space \(({\sX},d_{\sX})\) with Borel \(\sigma\)-algebra \(\bF_{\sX}\), is a metric space \(({\sB}_{\sX},d_P)\) (weak topology), where \(d_P\), the Prokhorov metric, is defined as
    \beq
    d_P(P_1,P_2) = \inf \{a: P_1(A) \leq P_2(A^a) + a \;\;\forall A \in \bF_{\sX} \text{ and vice versa}\}
    \eeq
    where \(A^a = \{s \in {\sX}: d(s,A) < a \}\) and \(d(s,A) = \inf\{d(s,s_A), s_A \in A\}\) For details see Chapter~2 of \citep{billingsley2013convergence}. Also, for any two random variables \(X,Y\) defined on \(\sX\) and \(\sY\), the conditional probability \(P(X | Y=y)\) is associated with a function \(\mu: \sY \rightarrow \sB_{\sX}\) (see lemma 1.37 and Chapter~5 in \citet{kallenberg2002foundations}). We shall call this function a probability kernel function.

    \subsection{Recent Techniques from Machine Learning}
     
    Now we present a brief review of two machine learning techniques recently proposed in literature that we have applied to our problem: distribution regression \citep{szabo2016learning} and transfer learning \citep{blanchard2011generalizing}.
    
    \paragraph{Distribution regression}
    Distribution regression \citep{szabo2016learning} is a technique to estimate the mappings from the space of distributions on a compact space \(\mathcal{X}\), \(\sB_{\sX}\) to  \(\sS\), a separable Hilbert space when the only access to the distribution is through samples drawn from it \citep{szabo2016learning}. Say we are given \(N\) training samples \(\set{\set{x_j^{(i)}}_{j=1}^{n_i}, s_i}_{i=1}^N\), drawn from a meta-distribution over \(\sB_{\sX} \times \sS\).
    The objective is to estimate a function \(r: \Phi(\sB_{\sX}) \rightarrow \sS\), where \(\Phi(\sB_{\sX})\) is the image of \(\sB_{\sX}\) under the mean embedding, such that
    \begin{equation*} 
    r^* = \argmin_{r \in {\sH}} \;\;{\E}[\n{r(\Phi(P_X)) - S}_\sS^2],
\end{equation*}
    where the mean embedding is defined as \(\Phi(P_X) = \int_{\sX} k(\cdot,x)dP_X\) for a kernel \(k\). The resulting regularized optimization and predictor is
    \begin{equation*}
        \begin{aligned}
            & \widehat{r}_{\xi_2} = \argmin \f{1}{N} \sum_{i=1}^N \n{r(\Phi(\widehat{P}_X^{(i)})) - s_i}_\sS^2 + \xi_2 \n{r}_{\sH}^2 \\
            \Rightarrow & \widehat{r}_{\xi_2} (\Phi(\widehat{P}_X)) = k_r (K + N \xi_2 I)^{-1}[s_1, s_2, \cdots, s_N]
\end{aligned}
\end{equation*}
    where \(K\) is the kernel matrix, \(k_r\) is the kernel vector of \(\Phi(\widehat{P}_X)\) with respect to the training distribution embeddings and \(\xi_2\) is the regularization variable \citep{szabo2016learning}. It has been shown that when the embedding is H\"{o}lder continuous with exponent h, this estimator is consistent and upper bounds for convergence can be obtained \citep{szabo2016learning}. Lastly, when space \(\sX\) is a Polish space, universal kernels that are dense in the space of continuous functions over compact metric spaces \citep{christmann2010universal,steinwart2002influence} can be used. Distribution regression will be used to perform estimation of orbital parameters.

    \paragraph{Marginal Prediction}
    Consider a Polish space \(\sX\), a binary classification space \({\sY} = \{-1,+1\}\), and a loss function \(L: \Rbb \times \sY \rightarrow \Rbb_{+}\). Let the space of distributions over \(\sX \times \sY\) be \(\sB_{\sX \times \sY}\). 
    Assume that there exists a distribution \(\lambda\) over \(\sB_{\sX \times \sY}\) such that \(\lambda = \lambda_{Y|X} \lambda_X\) and \(\lambda_{Y|X} = \delta_D\) almost everywhere, where \(\delta_D\) is the Dirac-Delta function. 
    For such functions, \(\exists h:\sB_{\sX} \times \sX \rightarrow \sY\) such that \(y=h(P_X,x)\). Let \({\sH}_{k}\) be the reproducing kernel Hilbert space (RKHS) associated with kernel \(k_1: \sX \times \sX \rightarrow \Rbb\). For \(\Phi(\sB_{\sX})\), the set of mean embeddings associated with \(\sB_{\sX}\), let \({\sH}_{k_P}\) be the RKHS associated with the kernel \(k_P: \Phi({\sB}_{\sX}) \times \Phi({\sB}_{\sX}) \rightarrow \Rbb\).  We seek an estimate \(h_{\sH}\) of \(h\) such that the following criterion is satisfied:
    \begin{equation*}
    h_{\sH} = \argmin_{h \in {\sH}_{\bar{k}}} {\E}_{P_{XY} \sim \lambda, (X,Y) \sim P_{XY}} [L(h(\Phi(P_X),X),Y)],
    \end{equation*}
    where \(\bar{k}: (\Phi({\sB}_{\sX}) \times {\sX}) \times (\Phi(\sB_{\sX}) \times \sX) \rightarrow \Rbb\) and \(\bar{k}((P_X,X)(P_{X^\prime},X^\prime)) = k_P(\Phi(P_X),\Phi(P_{X^\prime}))k_1(X,X^\prime)\).
    The resulting regularized problem has the formulation \citep{blanchard2011generalizing}
    \begin{equation*}
    \widehat{h}_{\xi_1} = \argmin_{h \in {\sH}_{\bar{k}}} \f{1}{N}\sum_{i=1}^N \f{1}{n_i}\sum_{j=1}^{n_i} L(h(\Phi(\widehat{P}_X^{(i)}),X_{ij}),Y_{ij}) + \xi_1 \n{h}_{{\sH}_{\bar{k}}}^2.
    \end{equation*}
    We shall use marginal prediction (or marginal transfer learning) to perform classification of feature vectors of multiple spacecraft.


    \section{Problem Setup} \label{sec:probsetup}
In its most generic case, the orbit determination problem of spacecraft using networked ground stations is essentially tracking an object with a network of sensors. Consider \(n_S\) objects with orbit parameters \(\{\Gamma_i\}_{i=1}^{n_S}\) drawn from the space \(\tilde{\sJ}^{n_S} = \sJ\) according to a probability distribution \(P_\Gamma\), which is known a priori and has compact support. 
    There are \(n_G\) ground stations acting as sensors.
The spacecraft system produces vectors \(F = \begin{bmatrix} F_1 & F_2 & \cdots & F_{n_G} & T_S \end{bmatrix}^T = \begin{bmatrix} \tilde{F} & T_S \end{bmatrix}^T\) over a set \(\sF = \tilde{\sF}\times \tilde{\sT}\) (datapoints in \(\tilde{\sF}\) generated over a time interval \(\tilde{\sT}\)). The connection between samples in the space \(\tilde{\sF}\) and the time random variable \(T_S\) is governed by a parameter \(z\) specific to the spacecraft and modulated by two related dynamic systems \(U\) and \(V\). Dynamic system \(V\) governs visibility of \(F\): for \(V(\gamma,t) = \begin{bmatrix} V_1(\gamma,t), V_2(\gamma,t), \cdots, V_{n_G}(\gamma, t) \end{bmatrix}\), if \(V_i(\gamma,t) \notin O_i, i \in \set{1,2, \cdots, n_G}\) then \(\tilde{F} = 0\). The dynamic system  \(U = [U_1, U_2, \cdots U_{n_G}]\) described by
    \begin{equation} \label{dsod0}
    \begin{aligned}
    \dot{\tilde{\gamma}} & = g_0(\tilde{\gamma}) \\
    \tilde{f}_i &= h_0(\tilde{\gamma},z) \\
    \tilde{\gamma}(0) &= \gamma,
    \end{aligned}
    \end{equation}
    governs the value of \(\tilde{F}_i\) when \(V_i(\gamma,t) \in O_i\). These observations are produced according to \(T_S \sim P(T_S|\Gamma = \gamma, Z=z)\) defined as
    \begin{equation*}
        P(T_S \in B | \Gamma = \gamma, Z=z) = P(T_S \in B | T_S \in \bigcup_{i=1}^{n_G} \set{t \in \tilde{\sT}: V_i(\gamma, t) \in O_i}, Z=z)
    \end{equation*}
    
    The unconditional distribution \(P(T_S | Z=z)\) is a characteristic of the system which is known. It is generally the case that the dynamic systems \(U\) and \(V\) have only partial differential equation based descriptions and no closed form descriptions. The PDE based equation can be used to draw samples of the evolution of \(U\). The sensors (ground stations) produce observations \(X = \begin{bmatrix} \tilde{X}_1 & \tilde{X}_2 & \cdots & \tilde{X}_{n_G} & T_{n_G} \end{bmatrix} \) of \(F\) distributed as \(P(X|F=f)\).
     
    For example, in a direction of arrival and range (DOAR) based orbit determination scenario for one spacecraft, \(\tilde{F}\) represents the theoretical noiseless vector of direction of arrival and range seen at the ground station at time \(T_S\), and \(\tilde{X}\) represents the observed Azimuth, Elevation and Range measurements at the ground station with timestamps \(T_G\).   U is the dynamic system of the DOAR, and \(z\) represents the parameters of the measurement system which are essential to draw samples of \(X\), such as noise characteristics and probability of measurement over \([0,T_{max}]\). The dynamic system \(V\) describes the elevation of the spacecraft with respect to the ground station such that we can perform measurements of U only if the elevation of the spacecraft at time t: \(V(\gamma)(t) \in [0,\pi/2]\). 
    
    With this scenario, the orbit determination problem can be stated as follows. Given \(P_\Gamma, P(T_S|z), U, V\) and \(P_{X|F}\) over the time interval \(\tilde{\sT}\) in a form sufficient enough to perform sampling, and observations \(\set{x_1,x_2,\cdots, x_{n_T}}\), which are samples of \(X\), we would like to estimate \(\{\gamma_i\}_{i=1}^{n_S}\), the orbital parameters. 
    
    \begin{remark}
    \begin{itemize}
        \item Note that in all but the most simplified case, closed form expressions for \(U\) and \(V\) are not known and are known only through differential equations and perturbation equations. Samples for \(U\) and \(V\) are drawn through propagators (which may be purely analytical or simplified). 
        \item We will present the theory for the orbit determination and classification with \(n_S = 2\) and point to techniques in literature which can be used to extend the algorithm to general \(n_S\).
        \item Here we assume that the feature vectors \(\set{x_1,x_2,\cdots, x_{n_T}}\) (including the time stamps of those feature vectors) are independent and identically distributed (i.i.d.) from a probability distribution known prior to generation of observations, even though the observations may be generated sequentially in time. While traditional treatment of dynamic system observations are as sample paths of random processes, we differ in two aspects: we consider only finite time treatment with random samples in time and we allow for multiple independent sets of sensors to produce \(P(X|F)\).
    \end{itemize}
    \end{remark}
   
    \section{Non-Parametric Orbit Estimation and Classification} \label{sec:mlod}
    We now present a mathematical framework and analysis of the above system.  
    
    \subsection{Framework} \label{prob_formulation}
    Consider the orbit determination scenario from Section \ref{sec:probsetup}. Given system parameters \(Z=z\) and the distribution of timestamps of observations \(P(T_S|z)\),  a probability distribution on the spacecraft output vectors, \(F\), is induced by the set of spacecraft orbit parameters, \(\Gamma = [\Gamma_1, \Gamma_2, \cdots, \Gamma_{n_S}]\). Samples of \(F\) generate samples of measurements, \(X\), at the sensor network. Note that the timestamps, \(T_S\) and \(T_G\), are not necessarily the same, especially when accounting for propagation delays through the channel.  This results in the graphical model (of the probability dependences) of the system as shown in Figure~\ref{fig:1}. 
    \begin{figure}[htbp]
        \centering
        \includegraphics[width=0.3\textwidth]{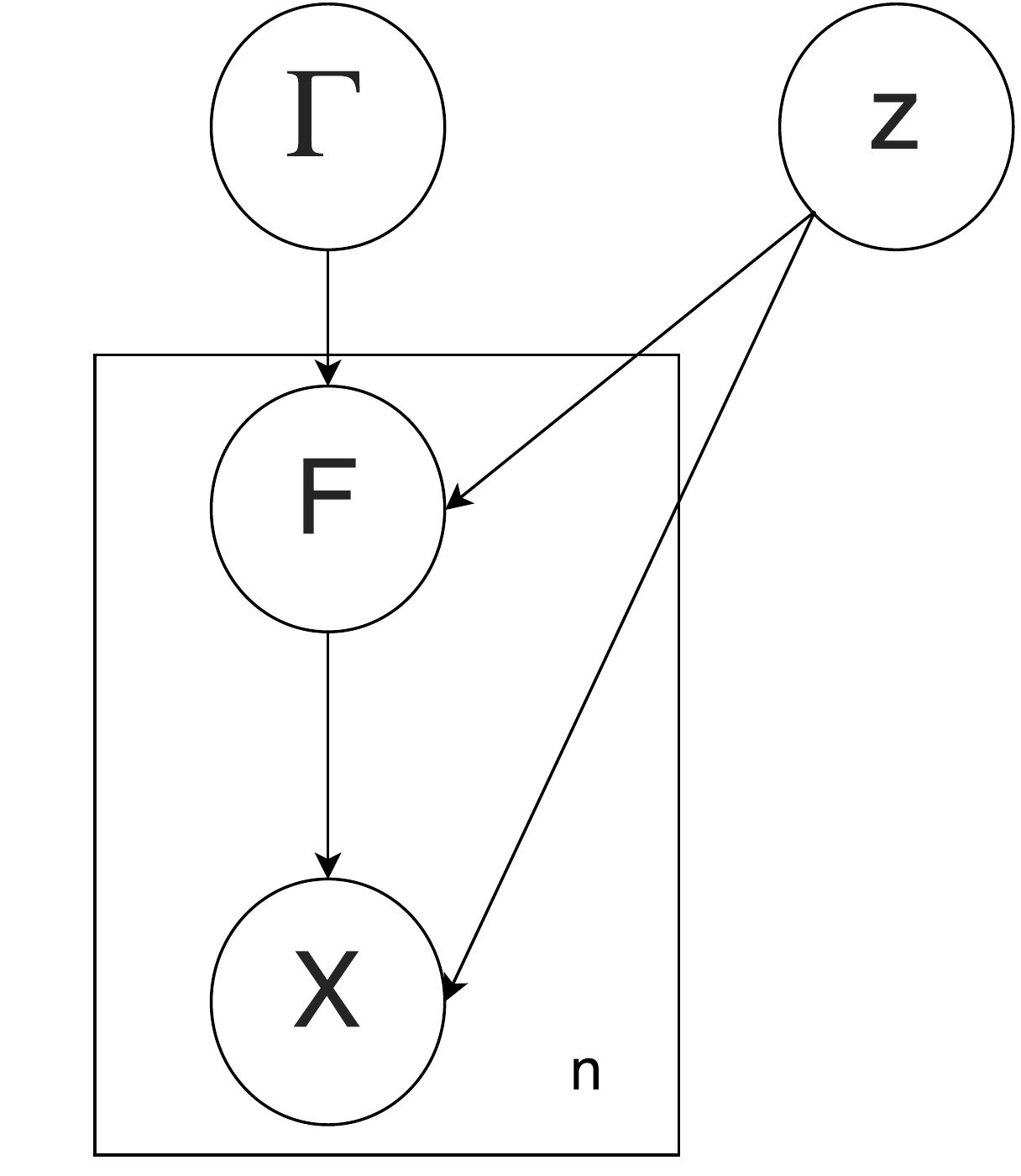}
            \caption{Graphical model of the orbit determination system.}
            \label{fig:1}
    \end{figure}
    
    Based on this, the probability distribution is split as
    
    \beq \label{pgmpdf1}
    P(\Gamma,F,X|z) = P(\Gamma)P(F|\Gamma,z)P(X|F,z).
    \eeq
    
    The conditional probability distribution \(P(F|\Gamma,z)\) is shaped by a deterministic non-linear dynamic model describing the system and operating on \(T_S\). 
    
    We define the observation function of the dynamic system as \(W: \sJ \rightarrow (\tilde{\sT} \rightarrow \tilde{\sF})\), which maps the initial conditions \(\sJ\) to the space of functions that in turn map from \(\tilde{\sT}\) to \(\tilde{\sF}\) such that \(W(\gamma)(t) = \onev_{V(\gamma,t) \in O} \circ U(\gamma,t)\) where \(\circ\) is the Hadamard or element-wise product. 
In the case of orbit determination, such a function is available through equation \eqref{dsod0} and models 
the orbit propagation due to gravity and other disturbance forces on a particular orbit, \(\gamma\).   For specific example forms of these 
astrodynamic equations, see \citet[chap.~9]{vallado2001fundamentals}.
Defining \(C\) and \(B\) as sets in the corresponding \(\sigma\)-fields of \(\tilde{F}\) and \(T_S\), and \(W(\gamma)^{-1}\) as the pre-image of \(W(\gamma)\), the conditional distribution \(P(F|\Gamma,z)\) of a system can now be written as
    \beq \label{pdfsystem}
    P(F \in C \times B | \Gamma = \gamma, z) = P(\tilde{F} \in C, T_S \in B|\Gamma=\gamma,z) = P(T_S \in B \bigcap W(\gamma)^{-1}(C)|\Gamma=\gamma,z),
    \eeq
    
    for \(\gamma \in \sJ\). The measure \(P(X|F)\) then produces noisy observations. 
 
    Given this system, we next present a mathematical analysis of it to provide insight into consequences of observability on the system. The distribution \(P(\Gamma)\) induces a probability measure \(\rho\) on \(R_\sX\), a set of probability measures on \(\sX\)(the set is \(\set{P(X|\Gamma=\gamma), \gamma \in \sJ}\)). \(R_\sX\) is a subset of \(\sB_\sX\), the set of all borel probability measures on \(\sX\). We show that under certain conditions, there exists a continuous map from \(R_\sX \subseteq \sB_\sX\) to \(\sJ\) which can describe the orbital parameters. This continuous map can be learnt from random samples of the probability distributions drawn as random samples from the space of probability distribution \(\sB_\sX\) using machine learning techniques and can then be used to estimate initial conditions from test datasets generated by spacecraft.

    \subsection{Mathematical Analysis} \label{subsec_mathanalysis}
    For the system defined in Equations \eqref{pgmpdf1} and \eqref{pdfsystem}, we make the following assumptions on the spaces \(\sJ, \sF, \sX, \tilde{\sT}\) and the probability distributions associated with them in order to characterize the effect of \(U\) on the probabilistic system:
    
    \begin{enumerate}[align=left, leftmargin=*, label=\textbf{A \Roman*}]
        \item \label{Pone} \((\tilde{\sF},d_{\tilde{F}}), (\tilde{\sX},d_{\tilde{X}})\) are compact metric spaces, \(\sJ\) is a compact subset of a real separable Hilbert space and \(\tilde{T}\) is a compact subset of \(\Rbb_+\) endowed with the regular Borel measure \(\Mm\).
        \item \label{Ptwo} \(P_{T_S}\) is absolutely continuous over support \(\sT \subseteq \tilde{\sT}\).
        \item \label{Aone} The probability kernel function from \(\sF \rightarrow \sP(\sX)\), for \(\sP(\sX) \subseteq \sB_{\sX}\), is a homeomorphism in the weak topology induced by the Prokhorov metric \((\sB_\sX,d_P)\).
    \end{enumerate}
    Let \(\bF_\sX\) be the \(\sigma\)-algebra induced by open (or closed) balls on \(\sX\) (Borel \(\sigma\)-algebra).

    Assumption \ref{Pone} limits the analysis of the system to those that are most suitable to characterization in terms of probability measures, which is most systems of interest. Assumption \ref{Ptwo} requires the probability distribution of measurements over the observation time admits a density.  Assumption \ref{Aone} is required for noise characteristics of the system where we assume that if the underlying noiseless parameters (such as directional of arrival, Doppler change or radar measurements) change, so does the probability distribution of the measurement system and this change is continuous. As a working example for assumption \ref{Aone}, when one measures range rate with narrowband communication systems, it has been shown that the correlation function \(S(f)\) from which the feature vector \(X\) is obtained can be written as \(S(f) = Q(f_c+f_d)+\text{residual}\) \citep{gardner1987spectrala,gardner1987spectrald}, where \(f_c\) is the center frequency of the RF transmission and \(f_d\) is the Doppler. 
    
     We now present two theorems that detail the consequences of the assumptions on the framework. We define distance metrics on the set of functions defined on \(\sT\) to \(\tilde{F}\) as
    \begin{equation}
        d_{U,p}(g_1, g_2) = \bigg[ \int_{\sT} (d_{\tilde{F}}(g_1(t),g_2(t)))^p d\Mm \bigg]^{\frac{1}{p}}
     \end{equation}
     for \(1 \leq p < \infty\) and 
     \begin{equation}
         d_{U,\infty}(g_1, g_2) = \sup_{t \in \sT} d_{\tilde{F}}(g_1(t),g_2(t))
     \end{equation}
     when they exist. 
     The dynamic system \(U\) is said to be observable in \(\sT\) if there exists an inverse \(U^{-1}:(\sT \rightarrow \tilde{F}) \rightarrow \sJ\) unique almost everywhere. This definition of observability is more in line with identifiability and subsumes the definition of non-linear observability used in traditional settings.

     For a first step analysis, we ignore the effect of visibility \(V\) and consider the case where points in \(U(\gamma,t)\) for some \(\gamma \in \sJ\) and \(t \in \sT\) will always be observed.i.e., the observations aren't modulated by line of sight and ground station specific horizon considerations and samples are produced through out \(\sT\) i.e., \(P(T_S | \gamma, z) = P(T_S | z)\). In doing so, we can analyze the probability distributions \(\set{P(X|\Gamma = \gamma): \gamma \in \sJ}\) resulting from simple restrictions on the continuity of \(U\). 

    \begin{theorem} \label{thmfcn}
        For the system defined by equations \eqref{pgmpdf1} - \eqref{pdfsystem} with assumptions \ref{Pone} and \ref{Aone}, 
        if \(U\) is observable in \(\sT\) and Lipschitz continuous in \(d_{U,\infty}\) with respect to \(\gamma\) in \(\sJ\) then there exists  a continuous inverse mapping \(\lambda: R_{\sF} \rightarrow \sJ\) for a compact set \(R_{\sF} \subseteq \sB_{\sX}\) on the topology \(({\sB}_{\sX},d_P)\).
    \end{theorem}
    \begin{proof}
        
        If \(U\) is observable and continuous then there exists a continuous bijective mapping  from \(\sJ\) onto \(U(\sJ)\). For such a bijective mapping \(U\), let the probability kernel function associated with \(P(F \in \cdot|\Gamma=\gamma,z) = \mu(\gamma)\). Then, for a given \(\delta\) we have \(\epsilon\) such that: \(\gamma_1, \gamma_2 \in \sJ\) with \(\n{\gamma_1 - \gamma_2}_{\sJ} < \epsilon\) implies \(d_{\tilde{F}}(U(\gamma_1)(t), U(\gamma_2)(t)) < \delta\), \(\forall t \in \tilde{\sT}\). 
        
        Consider sets \(C \in {\bF}_{\tilde{\sF}}, B \in {\bF}_{\sT_S}\). For any \(t \in U(\gamma_1)^{-1}(C)\) we can find a point \(f \in U(\gamma_2)(C^\delta)\) such that \(d_{\tilde{F}}(U(\gamma_1)(t),f) <\delta\). This implies for every set \(D = C \times B\), we have \(\tilde{D} = C^{\delta} \times B \subseteq D^\delta\) such that \(\mu(\gamma_1)(D) = \mu(\gamma_2)(\tilde{D})\) and \(\mu(\gamma_1)(D) \leq \mu(\gamma_2)(D^\delta) + \delta\). Similarly \(\mu(\gamma_2)(D) \leq \mu(\gamma_1)(D^\delta)+\delta\). This implies that
        \begin{equation*}
        \begin{aligned}
        d_P(\mu(\gamma_1),\mu(\gamma_2)) &= \inf\{\alpha: \mu(\gamma_1)(D) \leq \mu(\gamma_2)(D^\alpha) + \alpha \text{ and } \mu(\gamma_2) \leq \mu(\gamma_1)(D^\alpha) + \alpha, \; \forall D \in \bF_{\sF} \} \\
        & < \delta
        \end{aligned}
        \end{equation*}
        
        Also, as \(U\) is observable over \(\sJ\), if \(\gamma_1 \neq \gamma_2\), then there exists \(D \in \bF_{\sF}\) such that \(\mu(\gamma_1)(D) \neq \mu(\gamma_2)(D)\). Therefore, there exists a continuous function from \(\tilde{R}_{\sF}\) to \(\sJ\) for \(\tilde{R}_{\sF} \subseteq \sB_{\sF}\). Since the kernel function from \(\sF\) to \(\sB_{\sX}\) is bijective and continuous, the hypothesis holds.
        
        
    \end{proof}

    Next, we introduce the effect of visibility only on the dynamics. We capitalize on the continuous and differentiable behavior of \(V\) seen in most astrodynamic systems. If \(V\) is differentiable and continuous then the corresponding indicator function \(\onev_{V(\gamma, t) \in O}\) will be continuous in \(d_{U,p}\) for some \(p\) when it exists. We consider the behavior of functions continuous in \(d_{U,p}\) to study the effect of \(V\) on \(W\). 
    \begin{theorem} \label{thmfcnlp}
        For the system defined by equations \eqref{pgmpdf1} - \eqref{pdfsystem} with assumptions \ref{Pone}, \ref{Ptwo} and \ref{Aone}, W is observable in \(\sT\) and continuous in \(d_{U,p}\) for \(1 \leq p < \infty\) if and only if there exists a continuous inverse map \(\lambda: R_{\sF} \rightarrow \sJ\) for a compact set \(R_{\sF} \subseteq \sB_{\sX}\) on the topology \(({\sB}_{\sX},d_P)\)
    \end{theorem}

    \begin{proof} 
        %
        %

        (\(\Rightarrow\)) Let \(\mu(\gamma)\) be the kernel function associated with \(P(F \in \cdot | \Gamma = \gamma, z)\) such that \(\mu(\gamma)(D) = P(F \in D | \Gamma = \gamma, z)\) for \(D \in \bF_{\sF}\). Let \(\nu(\gamma)\) be the density function of the probability measure \(\mu(\gamma)\) (This exists by assumption \ref{Ptwo}). Fix \(\epsilon^\prime >0\). For \(\epsilon^\prime\), pick \(\delta\) and \(\epsilon\) such that \(\epsilon^\prime > [\sup_{\gamma \in \sJ, t \in \sT} \nu(\gamma)(t)] \epsilon > \delta > 0\). By continuity of \(W(\gamma)\) in \(d_{U,p}\), for the given \(\delta, \epsilon\), \(\exists \delta^\prime\) such that \(\n{\gamma_1 - \gamma_2} < \delta^\prime \Rightarrow d_{U,p}(W(\gamma_1),W(\gamma_2)) < \delta \epsilon^{1/p}\). 
        
        Let \(\sT_\delta = \{t \in \sT: d_{\tilde{F}}(W(\gamma_1)(t),W(\gamma_2)(t))>\delta\}\). We have, for any such \(\gamma_1, \gamma_2\) defined above
        \begin{equation*}
            \begin{aligned}
                d_{U,p}(W(\gamma_1),W(\gamma_2)) &= \bigg[ \int_{\sT} d_{\tilde{F}}(W(\gamma_1),W(\gamma_2)) d\Mm \bigg]^{\frac{1}{p}} \geq \\
                                                 &\geq \bigg[ \int_{\sT_\delta} d_{\tilde{F}}(W(\gamma_1),W(\gamma_2)) d\Mm \bigg]^{\frac{1}{p}} \geq \delta \Mm(\sT_\delta)^{1/p}
            \end{aligned}
        \end{equation*}
        and \(d_{U,p} (W(\gamma_1),W(\gamma_2)) < \delta \epsilon^{1/p}\) implies \(\Mm(\sT_\delta)<\epsilon\) (similar to convergence in \(L_p\) implies convergence in measure). 

        Now, consider any set \(D = B \times C \in \bF_{\sF}\) (\(B \in \bF_{\sT}\) and \(C \in \bF_{\tilde{\sF}}\) since they are all Borel sigma algebras). 
        When \(\Mm(\sT_\delta)<\epsilon\), it follows from the definition of \(\sT_\delta\) that \(W(\gamma_1)^{-1}(C) \cap (\sT \setminus \sT_\delta) \subseteq W(\gamma_2)^{-1}(C^\delta) \cap (\sT \setminus \sT_\delta)\) and vice versa, where \(W(\gamma)^{-1}\) is the preimage of \(W(\gamma)\). 
        
        Additionally, we also have that \( P_T (B \cap W(\gamma_1)^{-1}(C) \cap \sT_\delta) \leq P_T (B \cap \sT_\delta) \leq P_T(\sT_\delta) = \tilde{\epsilon} \leq \tilde{\epsilon}+P_T (B \cap W(\gamma_2)^{-1}(C^\delta) \cap \sT_\delta)\). 
        Using the above arguments,
      \begin{equation*}
          \begin{aligned}
              \mu(\gamma_1)(D) &= P_T(B \cap W(\gamma_1)^{-1}(C)) \\
                               &= P_T(B \cap W(\gamma_1)^{-1}(C) \cap \sT \setminus \sT_\delta) + P_T(B \cap W(\gamma_1)^{-1}(C) \cap \sT_\delta) \leq \\
                               &\leq P_T(B \cap W(\gamma_2)^{-1}(C^\delta) \cap \sT \setminus \sT_\delta) + P_T(B \cap W(\gamma_2)^{-1}(C) \cap \sT_\delta) + \tilde{\epsilon} \\
                               &\leq P_T(B \cap W(\gamma_2)^{-1}(C^\delta) \cap \sT \setminus \sT_\delta) + P_T(B \cap W(\gamma_2)^{-1}(C) \cap \sT_\delta) + [\sup_{\gamma \in \sJ, t \in \sT} \nu(\gamma)(t)] \epsilon \\
                               &\leq P_T(B \cap W(\gamma_1)^{-1}(C^\delta)) + \epsilon^\prime \\
                               &\leq \mu(\gamma_2)(D^{\epsilon^\prime}) + \epsilon^\prime
          \end{aligned}
      \end{equation*}  
        and vice versa.
        Therefore \(d_P(\mu(\gamma_1), \mu(\gamma_2)) < \epsilon^\prime\).
        
        Also, as \(U\) is observable over \(\sJ\), if \(\gamma_1 \neq \gamma_2\), then there exists \(D \in \bF_{\sF}\) such that \(\mu(\gamma_1)(D) \neq \mu(\gamma_2)(D)\) (injective map).
        Since \(\sJ\) is compact, \(\mu\) forms a continuous and injective map to \(R_\sF \subseteq \sB_\sF\) and \((\sB_\sF, d_P)\) is a compact metric space (Prokhorov's theorem), we have that the image \(R_\sF = \mu(\sJ)\) is compact. Additionally, that there exists a continuous map \(\lambda: R_\sF \rightarrow \sJ\) (see \citet{rudin1964principles}) .

        (\(\Leftarrow\)) Proof by contradition. Assume there exists a homeomorphic map \(\lambda: R_\sF \rightarrow \sJ\) from a compact metric space \((R_\sF, d_P)\). Additionally, assume that equations \eqref{pgmpdf1} - \eqref{pdfsystem} hold and \(\lambda^{-1} = \mu\) almost everywhere, but \(U\) is not continuous for some particular \(\gamma_1\),i.e, \(\exists \epsilon >0\) such that for any ball \(B_\delta(\gamma_1, \n{\cdot}_2)\), \(W(B_\delta(\gamma_1, \n{\cdot}_2)) \nsubseteq B_\epsilon(W(\gamma_1),d_{U,p})\). This also implies the resulting mapping is not continuous in measure at \(\gamma_1\) and for some \(\epsilon^\prime > 0\), for any \(\delta > 0\), \(\mu(B_\delta(\gamma_1, d_{U,p}) \nsubseteq B_{\epsilon^\prime}(\mu(\gamma_1),d_P)\) (Following similar arguments as in the direct case). This, however, is a contradiction as \(\lambda^{-1}\) is continuous.
    \end{proof}

    \begin{remark}

        \begin{itemize}
            \item We are not assuming that the dynamic system itself be Borel measurable or in a metric topology. We are assuming that the resulting observation function have these characteristics. This is especially true in the case of a generating system based on Hamiltonian dynamics (using Poincar\'{e} elements for U) where the topology is locally lebesgue, but observations such as position and Doppler over \(n_G\) ground stations are mapped to a metric space.
            \item When set \(\sJ\) is upper bounded by \(C_\gamma\) as a consequence of assumption \ref{Pone}, the norm \(\n{\lambda}_{\rho}^2 \leq C_\gamma^2\), where \(\rho = \mu \circ P_\Gamma\) is the probability distribution induced on \(\sB_\sX\). This will be useful in section \ref{learningtheory}, for convergence rate analysis.
        \end{itemize}
    \end{remark}

    The above two theorems state that in a system where the time intervals of observation are independent of the initial conditions, when the system is observable and is continuous in certain metric spaces they are also continuous in the space of probability distributions seen by the observations \(B_\sX\). This allows us to work with probability distributions instead of the observation function of the dynamic system.

    However, in practical implementation of this system, unless in a very constrained setting, is it generally the case that generation of observations is dependent on the orbit, as the region of observations is limited by the horizon of the ground stations or their sensitivity in parts of the horizon. We address this scenario next.

    Now, consider scenario where the presence of observations is also governed by the state of the dynamic system \(V\). Let \(\sT(\gamma)\) denote the preimage of \(O\) with respect to a particular gamma:
    \beq \label{tgammadef}
    {\sT}(\gamma) = \bigcup\limits_{i=1}^{n_G} \set{t \in \tilde{\sT}: V_i(\gamma)(t) \in O_i}.
    \eeq
    It is to be noted that in this scenario all ground stations may not be able to generate observations, only the ground stations l for which \(V_l(t) \in O_l\) will produce observations. An example of this scenario would be when the elevation of the object with respect to the ground station is in \([0,\pi/2]\), for \(n_G\) ground stations. We will assume that in the event that at least one of the ground stations generate observations, the rest will generate an observation of zero and we shall work with the observation function \(W(\gamma) = \onev_{V(\gamma) \in O} \cdot U(\gamma)\) (Element-wise operation).  We shall model the conditional distribution \(P(F|\Gamma = \gamma, z)\) as
    
    \beq \label{pdfsystemV}
    \mu(\gamma)(D):= P(\tilde{F} \in C, T_S \in B|\Gamma=\gamma,z) = P(T_S \in B \bigcap W(\gamma)^{-1}(C)|T_S \in \sT(\gamma),\Gamma=\gamma,z),
    \eeq
    
    Using equation \eqref{pdfsystemV}, we can extend the above theorem to work with a system where the compact set of observations are a set of intervals (multiple satellite passes).  

    We modify the assumption on continuous distributions as follows
    \begin{enumerate}[align=left, leftmargin=*, label=\textbf{A \Roman*-A}]
        \setcounter{enumi}{1}
    \item \label{Atwo} \(P_{T_S | \Gamma=\gamma}\) is absolutely continuous over the support \(\sT(\gamma) \subseteq \sT, \forall \gamma \in \sJ\) 
    \end{enumerate} 
    
    \begin{corollary} \label{corollaryfcn}
        For the system defined by equations \eqref{pgmpdf1} - \eqref{pdfsystem}, \eqref{tgammadef} - \eqref{pdfsystemV} and assumptions \ref{Pone}, \ref{Atwo} and \ref{Aone}, if \(V\) continuous in \(\gamma\), differentiable in \(t\) and \({\mathfrak{m}}(\bigcup\limits_{i \in \{1,2,\cdots, n_G\}} \set{t: \f{\partial V_i}{\partial t} = 0})=0\) then the following are equivalent
        \begin{itemize}
            \item \(W\) is observable and continuous in \(d_{U,p}, 1 \leq p < \infty\) over \(\gamma \in \sJ\).
         \item there exists a continuous inverse map \(\lambda: R_{\sF} \rightarrow \sJ\) for compact \(R_{\sF} \subseteq \sB_{\sX}\) on the topology \(({\sB}_{\sX},d_P)\).
        \end{itemize}
    \end{corollary}
    
    \begin{proof}
        \(P(T_S | z)\) is absolutely continuous with respect to the Lebesgue measure (limited to the Borel \(\sigma\)-algebra). If \(\sT(\gamma)\) is continuous in \(\gamma\) and can be expressed as a union of intervals, we have for a given \(\delta\), \(\exists \epsilon_1, \epsilon_2\) such that \(\n{\gamma_1 - \gamma_2} < \epsilon_1 \Rightarrow d_S({\sT}(\gamma_1),{\sT}(\gamma_2)) < \epsilon_2 \Rightarrow d_P(\mu(\gamma_1), \mu(\gamma_2)) < \delta\). 
        
        Define \(\sT_i(\gamma) = \set{t \in \tilde{\sT}: V_i(\gamma)(t) \in O_i}\) for \(i \in \set{1,2, \cdots, n_G}\). Let \({\sT}_i^\epsilon(\gamma) = \set{t \in \tilde{\sT}: V_i(\gamma)(t) \in O_i^\epsilon}\) and \({\sT}_i^{-\epsilon}(\gamma) = \set{t \in \tilde{\sT}: V_i(\gamma)(t) \in O_i^{-\epsilon}}\) where \(O_i^\epsilon = \set{o \in {\Rbb}^n | d(o,O_i)< \epsilon}\) and \(O_i^{-\epsilon} = ((O_i^c)^\epsilon)^c\). Since \(V\) is continuous in \(\gamma\), we have \({\sT}_i^{-2\epsilon}(\gamma_1) \subseteq {\sT}_i^{-\epsilon}(\gamma_2) \subseteq {\sT}_i(\gamma_1) \subseteq {\sT}_i^{\epsilon}(\gamma_2) \subseteq {\sT}_i^{2\epsilon}(\gamma_1)\). For a given \(\gamma\), by definition, \({\sT}_i(\gamma) \subseteq {\sT}_i^\epsilon(\gamma)\). If the two sets \({\sT}_i(\gamma)\) and \(\sT_i^\epsilon (\gamma)\) are equal for all \(\gamma\), then the continuity condition is satisfied trivially and therefore we only need to consider the case when \(\sT_i(\gamma) \subset \sT_i^\epsilon(\gamma)\). For a given \(\epsilon\) consider \({\sT}_i^{-\epsilon}(\gamma) \Delta {\sT}_i^\epsilon (\gamma)\) i.e., the pre-image of \(O_i^{-\epsilon} \Delta O_i^{\epsilon} = \bigcup_{p \in Bd(O_i)} B_\epsilon(p)\). We have, from the definition of the Lebesgue measure, \(d_S( {\sT}_i^{-\epsilon}(\gamma), {\sT}_i^\epsilon (\gamma) ) \leq {\mathfrak{m}}(R_{O_i, \epsilon}) + {\mathfrak{m}} ( \bigcup_j C_j \cap V_i(\gamma)^{-1}(O_i^{-\epsilon} \Delta O_i^{\epsilon})) \), where \(R_{O_i, \epsilon}\) is the set where the derivative of \(V_i(\gamma)\) with respect to \(t\) is zero in \(O_i^{-\epsilon} \Delta O_i^\epsilon\) and \(C_j\) is a countable covering of the set \(\tilde{\sT} \setminus R_{O_i, \epsilon} \) over the neighborhoods of points where the implicit function theorem can be applied.
        
        Therefore, we have that when \({\mathfrak{m}}(R_{O_i, \epsilon}) = 0\), for any \(\epsilon_2 > 0\), \(\exists\) an \(\epsilon_3\) such that \(d_S(O_i^{-\alpha}, O_i^\alpha) < \epsilon_3\) implies \(d_S( {\sT}_i^{-\alpha}(\gamma) , {\sT}_i^{\alpha}(\gamma) ) < \epsilon_2\), which implies continuity of \(\sT_i(\gamma)\) and \(\sT(\gamma)\) with respect to \(\gamma\). The rest of the proof follows from theorem \ref{thmfcnlp}.
    \end{proof}
   
    
    Note that the condition of observability over \(\sT(\gamma)\) is significantly stronger than observability over \(\sT\). It is, however, a weaker assumption compared to observability at every \(t \in \sT\). A simple example for this assumption in low Earth orbits occurs when estimating orbits with Doppler. In cases when the right ascension of the ascending nodes differ by a small amount with all other parameters being identical including the ground stations, there will exist regions where the Doppler shifts are identical for significant sections of the two passes. They will, however, be observable as the point of zero Doppler will differ in time.

    The corollary essentially states that in the scenario with \(n_G\) ground stations producing observations, (observations which are generated from an i.i.d process over a time interval) a continuous map to the initial conditions exist from a compact subset of the space of probability distributions of the observation random variable \(X\) exists if two conditions are satisfied. First the observation function of the dynamic system is observable over the times when the probability of observations being generated are non-zero. Second, it is required that the rate of change of the visibility system \(V\) is non-zero almost everywhere. For the scenario where \(O\) represents the horizon and \(V(\gamma,t)\) represents the elevation of the spacecraft with respect to the ground station, corollary \ref{corollaryfcn} requires that for scenarios where the times of observation is modulated by the elevation, the rate of change of elevation with respect to the ground station is non-zero almost everywhere. This is guaranteed by Newton's laws of gravitation except for Geostationary orbits. However, for geostationary orbits \(\sT(\gamma) = \sT\) and the continuous map still exists according to theorem \ref{thmfcnlp}.
    In an orbit determination scenario with direction of arrival estimates, this implies that if the observability and continuity conditions in theorem \ref{thmfcn}, \ref{thmfcnlp} and corollary \ref{corollaryfcn} are satisfied (which is necessary for any estimator to be consistent), then there exists a continuous mapping from the probability distributions of the direction of arrival measurements observed to the orbital parameters. This continuous mapping also exists even when observations are spread across multiple ground stations in time \(\sT\).
    
    
    
    \subsection{Observability and Probability Distribution Sampling} 
    
    We shall now discuss how to non-parametrically estimate the continuous mapping from the space of probability distributions. While direct characterization of these probability distributions is prohibitive due to its complex nature, it is possible to draw samples from these distributions and estimate the kernel embedding associated with the observed probability distributions and perform regression over the kernel embeddings as detailed in \citet{szabo2016learning}. Since the mapping exists and is continuous over a complete topological space over a compact set, it is possible to estimate the function mapping the probability distribution of the RF transmissions to orbital parameters \(\gamma\) arbitrarily close to its estimate in the corresponding RKHS. When universal kernels are used, the RKHS will be dense in the space of continuous functions from \({R}_{\sX}\) to \(\sJ\).
    
    
    For the mathematical model described in Sections \ref{prob_formulation} and \ref{subsec_mathanalysis}, we are given \(n_T\) observations \(\set{x_j^T}_{j=1}^{n_T}\) and we would like to estimate 
    the orbital parameters \(\gamma\). 
    Generally, for probabilistic graphical models of this nature, Bayesian inference is applied, either through direct computation of the posterior probabilities, the Expectation Maximization algorithm, or through techniques such as Markov chain Monte Carlo inference methods \citep{lee2005nonlinear}. 
    However, in general astrodynamic cases with weak observability such as Doppler only orbit determination, there are two challenges. 
    First, there exists no closed-form expression of the conditional distributions. 
    Second, while it is possible to sample from the prior conditional distribution, direct parametric description of the posterior conditional is non-trivial, time variant and has significant noise. 
    However, it is possible to perform forward sampling, and from Theorem~\ref{thmfcn} there exists a continuous function from the space of probability distributions of \(X\) onto \(\Gamma\). 
    A recent technique proposed in machine learning literature, the two stage sampled distribution regression \citep{szabo2016learning} allows us to perform regression over the space of probability distribution in a consistent fashion to estimate the orbits. 
    Such a technique can be trained by generation of training data through information present on \(U, V, P_\Gamma\) and \(P_{X|F}\). 
    
    In the orbit determination scenario of estimating the orbits of \(n_S > 1\) spacecraft using \(n_G\) ground stations, the probability distribution obtained from the observations will be a mixture of the observations from the different spacecraft. 
    While direct estimation of the orbits is possible from the mixture, the noise due to the mixture can be significant which can slowdown rate of convergence. It also results in significant overhead in computation. This is because the embedding for the orbit of each satellite has to be computed from the full mixture distribution instead of points that are generated specifically for each spacecraft. Additionally, this projection from the mixture distribution to individual distributions can be highly non-trivial. To overcome this, we propose first classifying the embeddings of the mixture using generalized classification techniques such as \citep{blanchard2011generalizing}, obtaining the labels and then use the resulting embeddings for orbit determination of each individual satellite.
    
    Based on this, we propose the following methodology. Perform sampling to generate training data \(\set{\set{x_j^{(i)},y_j^{(i)}}_{j=1}^{n_i}, \gamma_{1,i}, \gamma_{2,i}, \cdots, \gamma_{n_S,i}}_{i=1}^N\). The resulting system satisfies the underlying assumptions of transfer learning \citep{blanchard2011generalizing} and distribution regression \citep{szabo2016learning}. Perform transfer learning to obtain an estimate of \(\set{y_j^T}_{j=1}^{n_T}\) and distribution regression to obtain estimates of \(\gamma\). The work by Blanchard et. al \citep{blanchard2011generalizing} provides results for classification in the scenario of \(n_S = 2\). This can be extended to the multiclass case with \(n_S > 2\) (See \citet{deshmukh2017multiclass}). We show in section \ref{learningtheory} that, when the transfer learning technique is universally consistent, this entire system will also consistently estimate the initial conditions of each object.
    
    The sampling can be performed from the probabilistic graphical model described in the mathematical formulation described in section \ref{subsec_mathanalysis}. It is crucial, however, that the probability distributions, especially with regards to bias, are samples as expected to be seen in the experimental data. This is not an impossibility as known sources can generally be characterized. It is also crucial that the orbit dynamical models are accurate in the limit for consistency requirements to hold. For Additional details, see Section~\ref{sec:results}.

    \section{Learning Theory} \label{learningtheory}
    
    We present consistency analysis of the proposed system consisting of the combination of the marginal predictor and the distribution regression system. We also present some remarks on the consistency behaviours when \(n_S=1\).

    \paragraph{Preliminaries} For the learning sytem proposed, there are three spaces of interest: \(\sX\) the space of observations, \(\sY\) the space of labels and \(\tilde{\sJ}^{n_S}:=\sJ\) the space of initial conditions. The marginal predictor has 3 kernels \(k: \sX \times \sX \rightarrow \Rbb, K:\phi(\sB_\sX) \times \phi(\sB_\sX) \rightarrow \Rbb\) and \(k^\prime:\sX \times \sX \rightarrow \Rbb\)(RKHS \(\sH_k, \sH_K\) and \(\sH_{k^\prime}\) with kernel \(k\) used for the embedding, \(K\) operates on the embedding produced by \(k\), \(\phi(\sB_\sX)\) and \(k^\prime\) operates on datapoints. The regressor has two kernels \(\bar{k}:\sX \times \sX \rightarrow \Rbb, \ck:\xi(\sB_\sX) \times \xi(\sB_\sX) \rightarrow \sL(\tilde{\sJ})\) (RKHS \(\sH_{\bar{k}}, \sH_{\ck}\)) with kernel \(\bar{k}\) used for the embedding and \(\ck\) operates on the embedding produced by \(\bar{k}, \xi(\sB_\sX)\). We denote by \(\Psi_K\) and \(\Psi_\ck\), the cannonical feature maps associated with \(K\) and \(\ck\). Similar to the notation in \citep{caponnetto2007optimal, de2005risk}, we shall denote the kernel operator associated with a kernel \(\ck\) at point \(\xi_P\), as \(\ck_{\xi_P}\). This corresponds to the mapping \(\ck_{\xi_P}: \tsJ \rightarrow \sH_\ck\) such that for \(\gamma \in \tsJ\)
    \begin{equation*}
        \ck_{\xi_P}\gamma = \ck(\cdot, \xi_P)\gamma
    \end{equation*}
    
    \(\rho(\phi(P_X),\Gamma)\) is the probability measure on the product space \(\sB_\sX \times \sJ\) induced by \(P_\Gamma\). \(\sL(L^2)\) is the space of linear operators on to \(L^2(\phi(\sB_\sX) \times \tsJ,\rho,\tsJ)\) (linear operators from \(\phi(\sB_\sX) \times \tsJ\) onto \(\tsJ\) that are square integrable in the measure \(\rho\)). Let \(C_\gamma\) be the upper bound on the norm of values in \(\tilde{\sJ}\). We denote by \(\widehat{f}\) the empirical estimate of any operator \(f\).

    In addition to \ref{Pone}, \ref{Atwo} (or \ref{Ptwo}) and \ref{Aone}, we make the following assumptions on the learning system. Note that some of these assumptions are more general than the assumptions in the previous section and bounds derived here should apply to section \ref{subsec_mathanalysis}. 
    \begin{enumerate} [align=left, leftmargin=*, label=\textbf{L \Roman*}]
        \item \label{Lone} \(\tsJ\) is a compact subset of a real separable hilbert space and \(\sX\) is a compact metric space.
        \item \label{Ltwo} Kernels \(k, k^\prime, \bar{k}, K\) and \(\ck\) are universal and bounded by constants \(B_k^2, B_{k^\prime}^2, B_{\bar{k}}^2, B_K, B_\ck\) respectively. In addition, the cannonical feature vectors associated with kernels \(K\) and \(\ck\), \(\Psi_K: \sH_k \rightarrow \sH_K\) and \(\Psi_\ck: \sH_{\bar{k}} \rightarrow \sH_\ck\) are H{\"o}lder continuous with constants \(\alpha\) and \(\beta\) and scaling factors \(\bL_K, \bL_\ck\) . i.e.,
            \begin{equation*}
                \forall v,w \in \sH_k \qquad \n{\Psi_K(v) - \Psi_K(w)}_{\sH_K} \leq \bL_K \n{v - w}_{\sH_k}^\alpha
            \end{equation*}
            and
            \begin{equation*}
            \forall v,w \in \sH_{\bar{k}} \qquad \n{\Psi_\ck(v) - \Psi_\ck(w)}_{\sH_\ck} \leq \bL_\ck \n{v - w}_{\sH_k}^\beta
            \end{equation*}
        \item \label{Lfour} The loss function \(l\) is classification calibrated \citep{bartlett2006convexity} and is \(L_l\)-Lipschitz in its first variable and bounded by \(B_l\).

    \end{enumerate}

    \begin{remark}
         We first present some observations regarding the convergence bounds of the single spacecraft scenario \(n_S = 1\). In this case, the system only consists of a two stage sampled regression system minimizing the objective function
            \begin{equation*}
                \sE(r) = \int_{\xi(B_\sX) \times \sJ} \n{r(\xi_{P_X}) - \gamma}_{\tsJ}^2 \; d\rho(\xi_{P_X},\gamma)),
            \end{equation*}
            to obtain the orbital parameters. These remarks tie into the analysis and convergence bounds in \citet{szabo2016learning}.
        
        \begin{itemize}
            \item The boundedness condition on \(\gamma\) results in the corresponding model error bound (see \citet{caponnetto2007optimal} and section 8 of \citet{szabo2016learning})
            \begin{equation*}
               \int_{\sJ} e^{\f{\n{\gamma - r_\sH(\xi_p)}_\tsJ}{M}} -  \f{\n{\gamma - r_\sH(\xi_p)}_\tsJ}{M} - 1 d\rho(\gamma | \xi_p) \leq \f{\Sigma^2}{2M^2}
            \end{equation*}
            The above bound will be non-zero when there are model errors such as errors due to inaccuracies in spherical harmonic coefficients of gravity perturbations, model errors introduced due to drag during tumbling and solar radiation pressure.

        \item Also, there exists \(r_\sH \in \sH_\ck\) such that 
           \begin{equation*}
               \sE(r_{\sH}) =\inf_{r \in \sH_{\ck}} \sE(r).
           \end{equation*}

        \item Additionally, if there exists a map \(\lambda:\sB(\sX) \rightarrow \gamma\) with \(\n{\lambda}_{\rho_P}^2 \leq C_\gamma^2\) (see remarks in section \ref{subsec_mathanalysis}) then 
        \begin{equation*}
            r_\sH(P) = \int_{\sB_\sX} \ck(\xi(P), \xi(P_X)) \; \lambda(P_X) \; d\rho_P(\xi(P_X)),
        \end{equation*}
         where \(\rho_P\) is the marginal of \(\rho\) with respect to \(P_X\). 

        \item If the following assumption on decay of the residuals of the spectral operator of kernel \(\ck\) defined as
        \begin{equation*}
            T = \int_{\xi(\sB_\sX)} \ck_{\xi_P} \ck_{\xi_P}^\ast d\rho(\xi(P))
        \end{equation*}
        holds, then the system belongs to the \(\sP(b,c)\) class of probability distributions and we can provide tighter upper bounds for \(\sE(r) - \sE(r_\sH)\): 

        \begin{enumerate} [align=left, leftmargin=*, label=\textbf{L \Roman*}]
            \setcounter{enumi}{3}
            \item Either \(\exists 1 \leq  b < \infty\) such that the eigen values of \(T\), \(t_i, i = 1,2, \cdots, N_{eig}\) behave such that 
            \begin{equation*}
                a_1 \leq i^b t_i \leq a_2 \qquad \forall i \geq 1
            \end{equation*}
            or \(N_{eig} < \infty\).
        \end{enumerate}

        \end{itemize}
    \end{remark}
 
    For \(n_S > 1\), the learning system consists of a marginal predictor followed by a two stage sampled regressor. The output of the marginal predictor is fed to the regressor for learning mapping for orbit determination of the corresponding labeled spacecraft (this is required for consistency as labels aren't available in test scenarios). The marginal predictor minimizes the average generalization error defined as \citep{blanchard2011generalizing}
    
    \begin{equation} \label{TLoptfcn}
        I(g,n_T) :=
            \E_{P_{XY}^T \sim \rho(P_{XY})}
                \E_{(X,Y)^T \sim (P_{XY}^T)^{\otimes n_T}} \bigg[
                \frac{1}{n_T} \sum_{i=1}^{n_T} l(g(\widehat{P}_X^T,X_i^T),Y_i)
                \bigg].
    \end{equation}
    We also define 
    \begin{equation*}
        g^\ast = \argmin I(g, \infty) = \argmin \E_{P_{XY}^T \sim \rho(P_{XY})}
                \E_{(X,Y)^T \sim (P_{XY}^T)^{\otimes n_T}} \bigg[ l(g(P_X^T,X^T),Y) \bigg],
    \end{equation*}

    Additionally, let \(\sR(g, P_{X,Y},l) = \E_{(X,Y) \sim (P_{XY})}[l(g(P_X,\cdot),Y)]\), the risk for a given probability distribution and loss \(l\). When the transfer learning system is classification calibrated \citep{bartlett2006convexity}, \(g^\ast\) will be equal the Bayes classifier in terms of risk. The points classified by the marginal predictor per dataset (or orbit) are then embedded into a second RKHS. We consider these reconstructed embeddings of the different classes at the output of the marginal predictor. Define

    \begin{equation} \label{TLfcnlabel}
        \bar{h} = [h_y]_{y \in \sY}
    \end{equation}
    such that
    \begin{equation} \label{TLfcndef}
        h_y(\phi_{P_X}) = \int_\sX \bar{k}(x,\cdot) \onev_{sgn(g(P_X,x))=y} dP_X,
    \end{equation}
    as the embedding of the classified data. 

    Consider one of the labels: \(h = h_y, \gamma = \tilde{\gamma}_y\) for any \(y \in \sY\), with \(h^\ast\) the embedding for \(g^\ast\). We perform two stage sampled regression on the outputs of \(h\) to minimize the error
    \begin{equation} \label{DRTLoptdef}
    \sE(r \circ h) = \int_{\phi(B_\sX) \times \sJ} \n{r \circ h(\phi_{P_X}) - \gamma}_{\tsJ}^2 \; d\rho(\phi_{P_X},\gamma))
    \end{equation}
 
    For the system described by equations \eqref{TLoptfcn} to \eqref{DRTLoptdef} we make the following assumptions
    \begin{enumerate}[align=left, leftmargin=*, label=\textbf{S \Roman*}]
       \item \label{Sone} There exists a function \(r_{\sH_\ck}\) such that
            \begin{equation*}
            \sE(r_{\sH} \circ h^\ast) =\inf_{r \in \sH_{\ck}} \sE(r \circ h^\ast)
            \end{equation*}
    \end{enumerate}
    
    %
    %

    Note that Assumptions \ref{Sone} and \ref{Atwo} interact with each other. This leads to either an implicit assumption regarding the Bayes classifier, or a modification of \(\sT(\gamma)\) when applied to the orbit determination system. When the Bayes classifier can classify data with zero error this implies that \(P_{Y|X}\) is \(\rho_{XY}\) almost surely a function of \(P_X\) and therefore \(\xi(P_{X | y}) = h_y^\ast(P_X)\) a.e. When the Bayes' error is non-zero, this implies that the underlying dynamic system is observable over the support posterior probability distribution (if it's not observable then \ref{Sone} will not hold). This will result in corollary \ref{corollaryfcn} holding with a modified support.
    
    We are interested in the convergence,
    \begin{equation*}
        \sE(\hat{r} \circ \hat{h}) - \sE(r_\sH \circ h^\ast) \rightarrow 0
    \end{equation*}
    
    We define the linear operator \(A_h: \sH_\ck \rightarrow L^2(\phi(\sB_\sX) \times \tsJ,\rho(\phi_{P_X},\gamma), \tsJ)\) such that for \(r \in \sH_\ck\)
    \begin{equation*}
        A_h r (\phi_p,\gamma) = \ck_{h(\phi_p)}^\ast r
    \end{equation*}
    This essentially implies that
    \begin{equation*}
        A_h r (\phi_p,\gamma) = r \circ h(\phi_p)
    \end{equation*}
    \(A_h\) is the canonical injection of \(\sH_\ck\) under the transformation \(h\). 
    
    The proof strategy is as follows: 
    \begin{enumerate}
        \item We first derive the conditions for consistency of \(A_{\hat{h}}\) the embedding of the classfied output under infinite sample settings in theorem \ref{th:lt1}. 
        \item Next, we extend on the conditions to derive rates for finite sample settings of the spectral operator \(T_h\) in its finite sample sense in lemma \ref{lem:lt2}. 
        \item We then use these arguments to provide high probability upper bounds for \(\sE(\hat{r} \circ \hat{h}) - \sE(r_\sH \circ h^\ast)\) in theorem \ref{th:ltconsistency}.
    \end{enumerate}

    Before we consider convergence properties of the system we present extensions of important theorems in \citet{caponnetto2007optimal}. 
    The proofs are straightforward and follow the same line of arguments as presented in \citet{caponnetto2007optimal, de2005risk}.

    \begin{theorem} \label{th:defAT}        
        If assumptions \ref{Lone} and \ref{Ltwo} hold for the given system, \(A_h\) is a bounded operator from \(\sH_\ck\) to \(L^2(\phi(\sB_\sX) \times \tsJ,\rho, \tsJ)\), the adjoint \(A^\ast: L^2(\phi(\sB_\sX) \times \tsJ,\rho, \tsJ) \rightarrow \sH_\ck\) is 
            \begin{equation*}
                A^\ast_h s = \int_{\phi(\sB_\sX) \times \sJ} \ck_{h(\phi_p)} s(\phi_p,\gamma) d\rho(\phi_p,\gamma)
            \end{equation*}
        where the integral converges in \(\sH_\ck\) and \(A_h^\ast A_h\) is the Hilbert-Schmidt operator on \(\sH_\ck\):
            \begin{equation*}
                T_h = \int_{\phi(\sB_\sX)} T_{h(\phi_p)} d\rho_P(\phi_p)
            \end{equation*}
        for \(T_{h(\phi_p)} = \ck_{h(\phi_p)} \ck_{h(\phi_p)}^\ast\)
    \end{theorem}

    \begin{theorem} \label{th:mincrit}      
        If assumptions \ref{Lone}, \ref{Ltwo} and \ref{Sone} hold, \(r_{\sH}\) is a minimizer of expected risk \(\sE( \cdot)\) under the composition map \(h^\ast\) iff it satisfies
            \begin{equation*}
                T_{h^\ast} r_{\sH} = A^\ast_{h^\ast} \gamma
            \end{equation*}
    \end{theorem}

    The above theorem is well known in linear algebra and is a consequence of the projection theorem.

    We shall first present bounds on \(A_h\) for the transfer learning system before we move on to the bounds of the complete system. This following theorem states that convergence of \(A_{\hat{h}}\) depends on whether the Bayes' decision boundary \(\sgn(g^\ast)\) is well defined. If the regions where the Bayes' decision boundary is ill defined has a measure greater than zero then convergence is not guaranteed. This also implies that under those conditions, the map \(\lambda\) may not be continuous for the class under consideration. We directly provide arguments for universal consistency.

    \begin{theorem} \label{th:lt1}
        For the learning system defined in equations \eqref{TLoptfcn}-\eqref{TLfcndef} assume that the conditions \ref{Lone} - \ref{Lfour} are satisfied. Then, if there exists a \(c \in [0,0.5]\) and a sequence \(\epsilon_j \rightarrow c\) such that 
        \begin{equation*}
            \tilde{\epsilon} = \lim_{\epsilon_j \rightarrow c} \E_{P_X \sim \rho, X \sim P_X} \big[\onev_{\set{|\eta(X)-0.5|<\epsilon_j}}\big] = 0,
        \end{equation*}
        then,
        \begin{equation*}
            \n{(A_{\hat{h}} - A_{h^\ast})}_{\sL(L^2)} \rightarrow 0
        \end{equation*}
        as
        \begin{equation*}
            I(\hat{g},\infty) \rightarrow \inf_{g \in C(\sB_\sX \times \sX \rightarrow \Rbb)} I(g,\infty)
        \end{equation*}
    \end{theorem}
    The proof is presented in Appendix \ref{proof:th:lt1}.
    
    The above theorem essentially states that as long as the bayes decision boundary is continuous and well defined over the space \((P_X, X)\), the embedding of the probability distributions with classification calibrated versions of transfer learning will converge to those of the bayes classifier. For example, this will hold for the system stated in section \ref{subsec_mathanalysis} when the bayes error is zero (in which case c=0.5 and \(\epsilon_j\) is the trivial sequence). When the bayes error is non-zero, as long as the regions where the bayes decision boundary is ill defined have measure zero, the embeddings should converge. Next we shall prove a similar result for the empirical spectral operator \(\hat{T}\) generated by data.

    \begin{lemma} \label{lem:lt2}
        For the learning system defined in equations \eqref{TLoptfcn}-\eqref{TLfcndef} assume that the conditions \ref{Lone} - \ref{Lfour} are satisfied. Then, if there exists a \(c \in [0,0.5]\) and a sequence \(\epsilon_j \rightarrow c\) such that 
        \begin{equation*}
            \tilde{\epsilon} = \lim_{\epsilon_j \rightarrow c} \E_{P_X \sim \rho, X \sim P_X} \big[\onev_{\set{|\eta(X)-0.5|<\epsilon_j}}\big] = 0,
        \end{equation*}
        and \(\epsilon_j^{2\beta} \sqrt{N_{dr}} \rightarrow \infty\) then,
        \begin{equation*}
            \n{(\hat{T}_{\hat{h}} - \hat{T}_{h^\ast})} \rightarrow 0
        \end{equation*}
        as
        \begin{equation*}
            I(\hat{g},\infty) \rightarrow \inf_{g \in C(\sB_\sX \times \sX \rightarrow \Rbb)} I(g,\infty)
        \end{equation*}
    \end{lemma} 
    The proof is presented in Appendix \ref{proof:lem:lt2}.

    Next we use theorem \ref{th:lt1} and Lemma \ref{lem:lt2} to analyze the consistency behavior of the entire system.

    \begin{theorem} \label{th:ltconsistency}
        For the system defined in equations \eqref{TLoptfcn}-\eqref{DRTLoptdef} assume that the conditions \ref{Lone} - \ref{Lfour} are satisfied. If there exists a \(c \in [0,0.5]\) and a sequence \(\epsilon_j \rightarrow c\) such that 
        \begin{equation*}
            \tilde{\epsilon} = \lim_{\epsilon_j \rightarrow c} \E_{P_X \sim \rho, X \sim P_X} \big[\onev_{\set{|\eta(X)-0.5|<\epsilon_j}}\big] = 0,
        \end{equation*}
        then, \(\sE(\hat{r} \circ \hat{h}) \rightarrow \sE(r_\sH \circ h^\ast)\) as 
        \begin{equation*}
            I(\hat{g},\infty) \rightarrow \inf_{g \in C(\sB_\sX \times \sX \rightarrow \Rbb)} I(g,\infty)
        \end{equation*}
        and
        \begin{equation*}
            \sE(\hat{r} \circ h^\ast) \rightarrow \sE(r_\sH \circ h^\ast)
        \end{equation*}
    \end{theorem}
    The proof is presented in Appendix \ref{proof:th:ltconsistency}.
    
    Next we present details of implementation of the system including the details of the astrodynamics modeling, some experimental results and synthetic dataset results for different scenarios. 
    
    \section{Results and Discussion} \label{sec:results}
    We consider four scenarios to test different aspects of the orbit determination system. The first is based on Doppler only orbit determination and the last three on direction of arrival and range information. We present results with Doppler information collected by cognitive radio based algorithms on software-defined ground stations from on-orbit transmissions of MCubed-2 and GRIFEX spacecraft testing algorithmic behaviour with high-noise low observability conditions. We then discuss the results for an example on-orbit deployment scenario of two spacecraft testing identification and orbit determination of satellites in a TLE lottery. The third scenario considers a lunar orbit where we test algorithmic behaviour with a chaotic system. In the fourth scenario, we  perform a comparison of a traditional orbit determination system based on the EKF with the proposed machine learning technique. Last, we shall discuss and provide comparisons of the different scenarios.
    
    We present results with one and two ground stations. The mathematical theory is broad enough to allow for networked ground stations with multiple types of sensors, however, we shall leave this for future work. We begin with details of the system architecture for single ground station scenarios. 
    
    \subsection{System Architecture} \label{sub:sysarch}
    The general architecture of the sampling and the orbit estimation system is shown in Figure \ref{fig:sysarch}. The prior \(P_\Gamma\), which represents the uncertainty of orbit parameters for learning, can be constructed either from launch characteristics and launch sequencing or from uncertainties in pre-launch TLEs. There are no limiting factors to \(P_\Gamma\) than those described in section \ref{subsec_mathanalysis}. It is necessary for the orbital elements used in this system to have parameters that are independent of each other in order to reduce computation requirements for training in the orbit determination step, since the kernel operator can be diagonal. The time sampling characteristics \(P_{T|z}\) and the noise characteristics \(P_{X|F}\) of the measurement system \(z\) must be estimated prior to generation of training data for the orbit determination. This estimation will depend on the deployment scenario. We provide examples of this for the Doppler only orbit determination technique in section \ref{ssub:DopplerOD}. 
    
    Sample generation is split into two subsystems: the propagator and the observer. The propagation system (or propagator) generates samples of the dynamic systems \(U\) and \(V\) at sample time points. The dynamic system must be unbiased in its generation of data and its error must be bounded. This holds as long as the errors in the spherical harmonic coefficients of the gravity model are bounded and does not have a constant bias error. We present test scenarios with two propagators: SGP4 and an analytical propagator with a numerical integration-based set up. The propagator sample time points are generated disregarding horizon \(O\) and sensitivity information of the ground station. The observer system then combines horizon information present in \(V, O\), the sensitivity information, and noise to generate samples \(X\) from \(P(X | \Gamma = \gamma)\) and the spacecraft id labeled \(1,\cdots, n_S\) as detailed in section \ref{subsec_mathanalysis}.

    The learning system depends on the number of spacecraft. In the single spacecraft scenario, this training information is then fed to the two stage sampled regression for orbit determination. In the multiple spacecraft scenario, \(N_{tl}\) orbit distributions are first used to train the marginal predictor. The rest of the samples are then classified using the marginal predictor and then used to train an \(n_S\) bank of regressors, one for each spacecraft. As shown in Section \ref{learningtheory}, this is necessary in the scenarios where identification of the spacecraft is not straightforward. Due to the sparsity and the convergence characteristics of the marginal predictor, \(N_{tl}\) is generally significantly lower than \(N_{dreg}\). Besides the serial behavior in training for the marginal predictor and the regressor, the system is entirely parallelizable. In fact, even though training for \(n_S\) regressors have to be performed, the number of kernel evaluations are equal to that of a single large regressor with all the data points. For the transfer learning system, to speed up evaluation, we used a random Fourier feature based transfer learning system proposed in \citet{blanchard2017domain}. 
    The distribution regression system consisted of 13 hyper parameters: 6 for the kernel bandwidths of the embeddings, 6 for higher RKHS \(\ck\) (as in equation \eqref{eq:kernelmatrix}) and one for the regularizer. 
    \begin{equation} \label{eq:kernelmatrix}
        \ck = 
        \begin{bmatrix}
            \ck_{\tilde{\gamma}_1} & 0 & 0 & 0 & 0 & 0 \\
            0 & \ck_{\tilde{\gamma}_2} & 0 & 0 & 0 & 0 \\
            0 & 0 & \ck_{\tilde{\gamma}_3} & 0 & 0 & 0 \\
            0 & 0 & 0 & \ck_{\tilde{\gamma}_4} & 0 & 0 \\
            0 & 0 & 0 & 0 & \ck_{\tilde{\gamma}_5} & 0 \\
            0 & 0 & 0 & 0 & 0 & \ck_{\tilde{\gamma}_6}
        \end{bmatrix}
    \end{equation}

    Grid search with 5 fold cross validation was used for training. The Michigan High-Performance Cluster was used for training and testing.
    Preprocessed orbit feature vectors can then be fed into this system to perform orbit determination of the set of spacecraft. The preprocessing steps are dependant on the type of feature vectors used for orbit determination.
    In this paper, in addition to the experimental results presented for the two spacecraft (MCubed-2 and GRIFEX), we also generate additional identically distributed sampling data to test performance of the system. For the direction of arrival and range (DOAR) systems we provide only synthetic results with data generated from analytical propagators.

    \begin{figure} [htbp] 
        \centering
        \includegraphics[width=\linewidth]{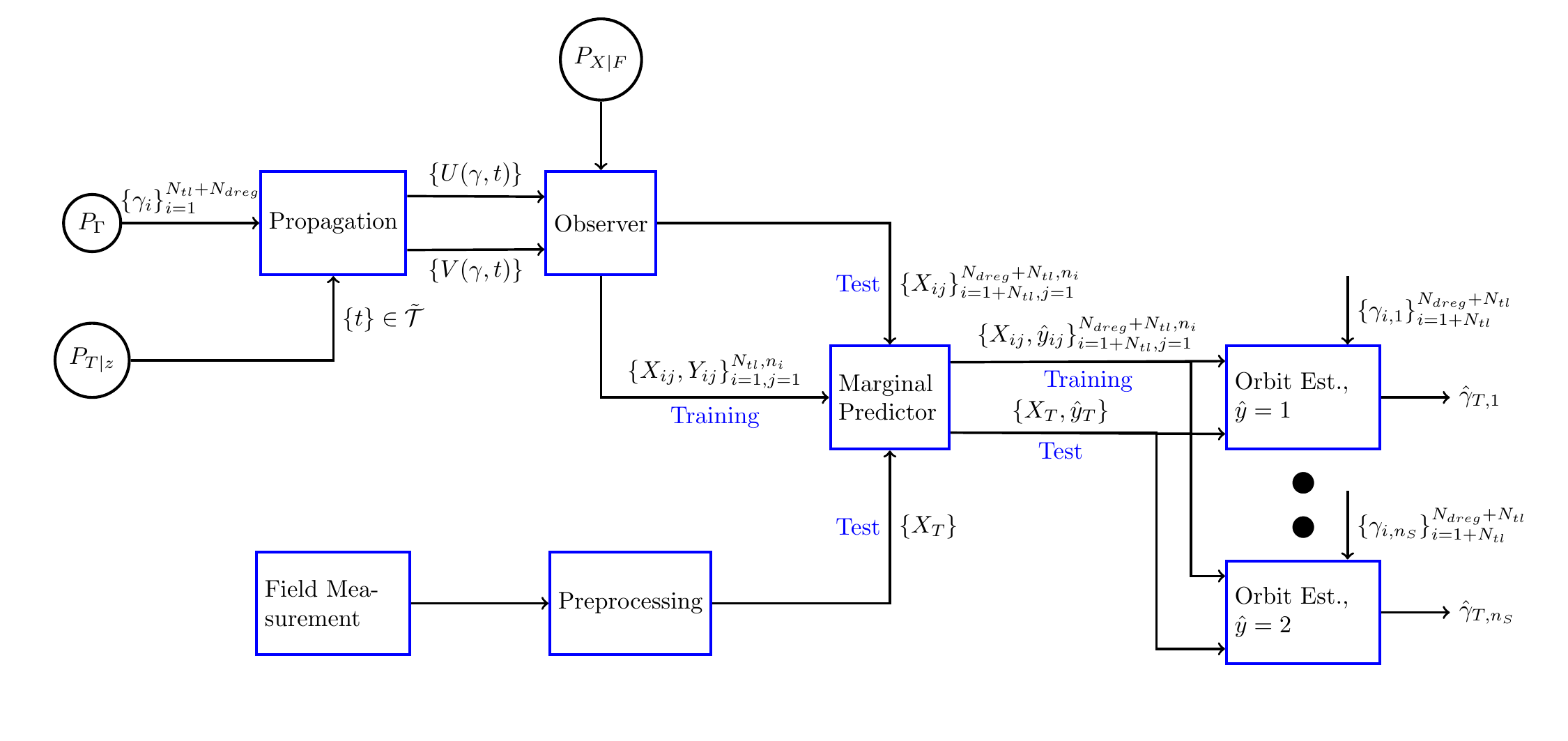}
        \caption{System Architecture}
        \label{fig:sysarch}
    \end{figure}

    \subsection{Doppler Only OD}
    \label{ssub:DopplerOD}
    The approach detailed in Section 4 states that if observability criteria are satisfied, then Doppler information alone should be sufficient to perform orbit determination of spacecraft. While analytical verification of observability for Doppler based observations is highly complex and non-trivial, the approach detailed can be used to verify observability through the performance of the learning system. 
    
    We run orbit determination for two low Earth orbit spacecraft - MCubed/COVE-2 and GRIFEX (\citep{norton2013nasa, cutler2015grfx, norton2012spaceborne}). 
    Their orbits were determined over an interval spanning 4 passes for MCubed-2 and 3 passes for GRIFEX at the Ann Arbor ground station. This analysis was performed using extracted Doppler data from actual passes. Both satellites have UHF telemetry channels at ~437.485 MHz and transmissions at 9600 bps, GMSK modulated waveforms. These transmissions will only be decodable when received energy per bit over noise crosses 13dB. However, decodability and identifiability is not a requirement for the proposed orbit determination technique.
    
    Cognitive radio approaches in blind cyclostationary feature extraction \citep{bkassiny2012blind} were used to extract Doppler, time, and data rate information. These algorithms were applied over recordings of raw, high-rate sampled data from an experimental, software-defined radio (SDR) based ground station. Complex baseband recordings were made of satellite transmissions with this SDR system for passes over Ann Arbor over a 6 hour interval starting at 23:00:00 UTC on 9 Feb. 2016. 
    The recordings were limited to predicted intervals around passes based on training data due to large file sizes of the recordings. Figure \ref{fig:rawiq} shows an example recording for an MCubed-2 pass. Note the variation in received power is due to undamped oscillations in attitude on orbit (MCubed-2 has damping only on one axis). 
    
    Appropriate FIR filter banks were used on the baseband signal to filter the software defined radio harmonics and known constant frequency out of band RF transmissions such as those seen around the 380 second marks at 437.504 MHz. Note that we do not assume the presence of prior orbit information accurate enough to use directional antennas to decode signals associated with the spacecraft. A low-gain wide-beam antenna can be used to collect raw spectral information to extract parameters associated with transmissions. The raw RF baseband signals recordings also consisted of noise due to transmissions to the spacecraft from the ground station, which were eliminated using power thresholding (RF leakage for 500 W transmissions were at least 30 dB higher than beacons due to attenuation). CubeSat modulated telemetry transmissions at 9.6 kbps which was used to isolate Doppler of the spacecrafts \citep{gardner1987spectrald}. Due to trivial classification requirements, data was manually classified before feeding into the orbit determination system.

\paragraph{Bias Correction} For the learning algorithm to operate as expected the experimental and training data offsets should be identical. However, due to implementation issues, there were specific communications system characteristics both on the spacecraft and on the ground station which resulted in bias in the recorded data. On the satellite, temperature variations and imperfect frequency calibration transmission center frequency led to frequency bias.  On the ground, there was a varying initial timing bias during the initialization of each recording (one recording per pass during the 6 hour interval).  This was due to coding inefficiencies and speed in loggine the large data (~3 gigabytes for 10 minutes) to the file system.  
    
    The frequency bias was corrected with two frequency offset corrections - one for MCubed-2 and one for GRIFEX.  The offsets were corrected by computing frequency offsets of a prior pass with similar spacecraft temperatures. 
The timing bias was corrected with 7 time offset corrections for the 7 recording intervals (4 for MCubed-2 and 3 for GRIFEX). 
    The time offsets were corrected by computing the time offset with respect to the TLE of the spacecrafts to align the points of maximum Doppler. The time offset corrections varied from 0 to 8 seconds. No other changes to the recordings were performed prior to extraction of features. Figure \ref{fig:Dopplerfeatures} shows an example of the extracted features and the JSpOC (Joint Space Operations Command) TLE post bias correction. 
    
    We do not expect to face this bias correction issue in future deployments of the orbit determination technique as center frequency behavior will be characterized prior to launch and the cognitive radio algorithms will be integrated into real-time operational software instead of being implemented over recordings in this experimental fashion.

    \begin{figure} [htbp] 
        \centering
        \includegraphics[width=0.7\linewidth]{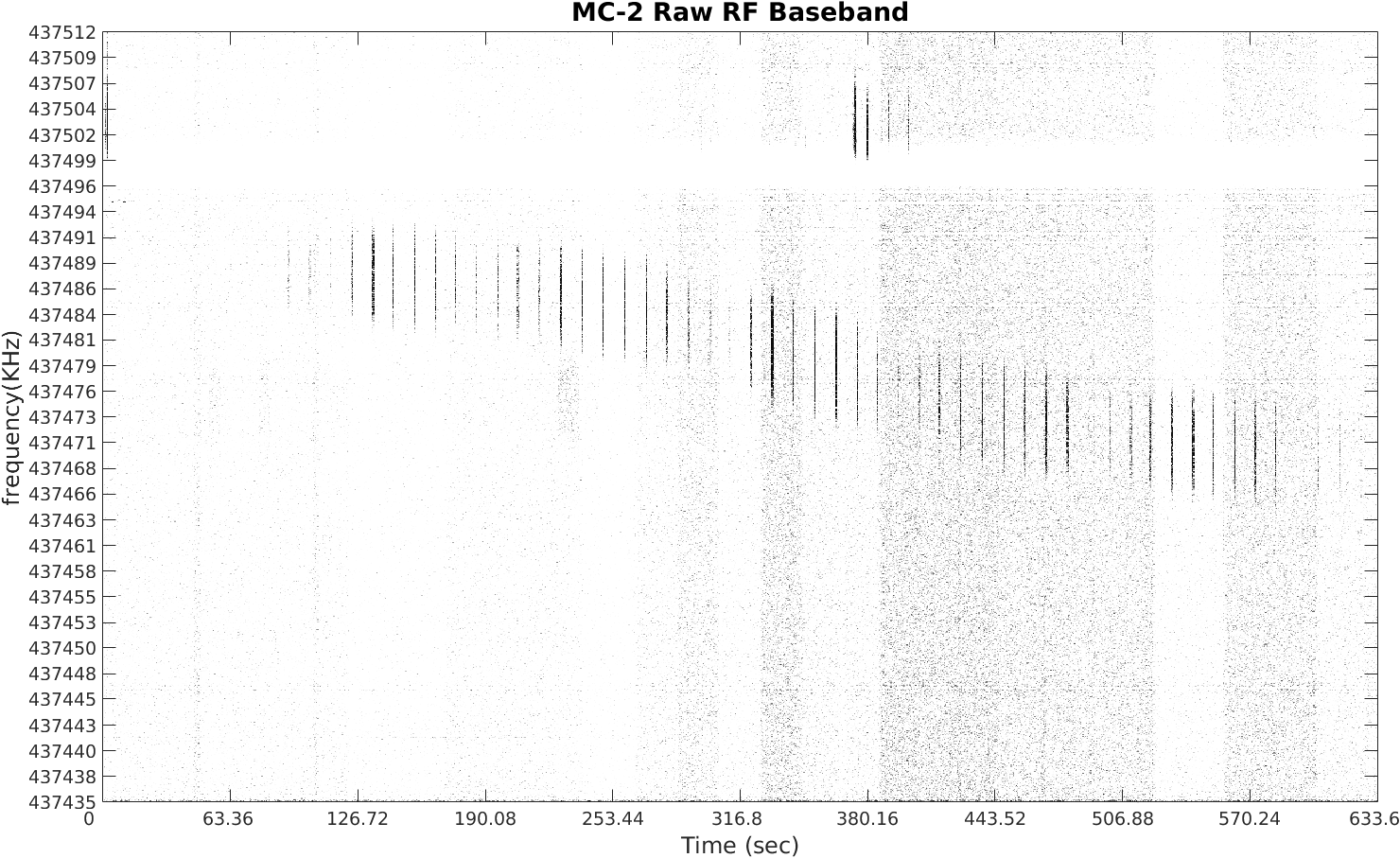}
        \caption{MCubed-2 Raw RF Baseband Recordings from 02-10-2016 at 01:20:24 UTC}
        \label{fig:rawiq}
    \end{figure}
    
    \begin{figure} [htbp] 
        \centering
        \includegraphics[width=0.6\linewidth]{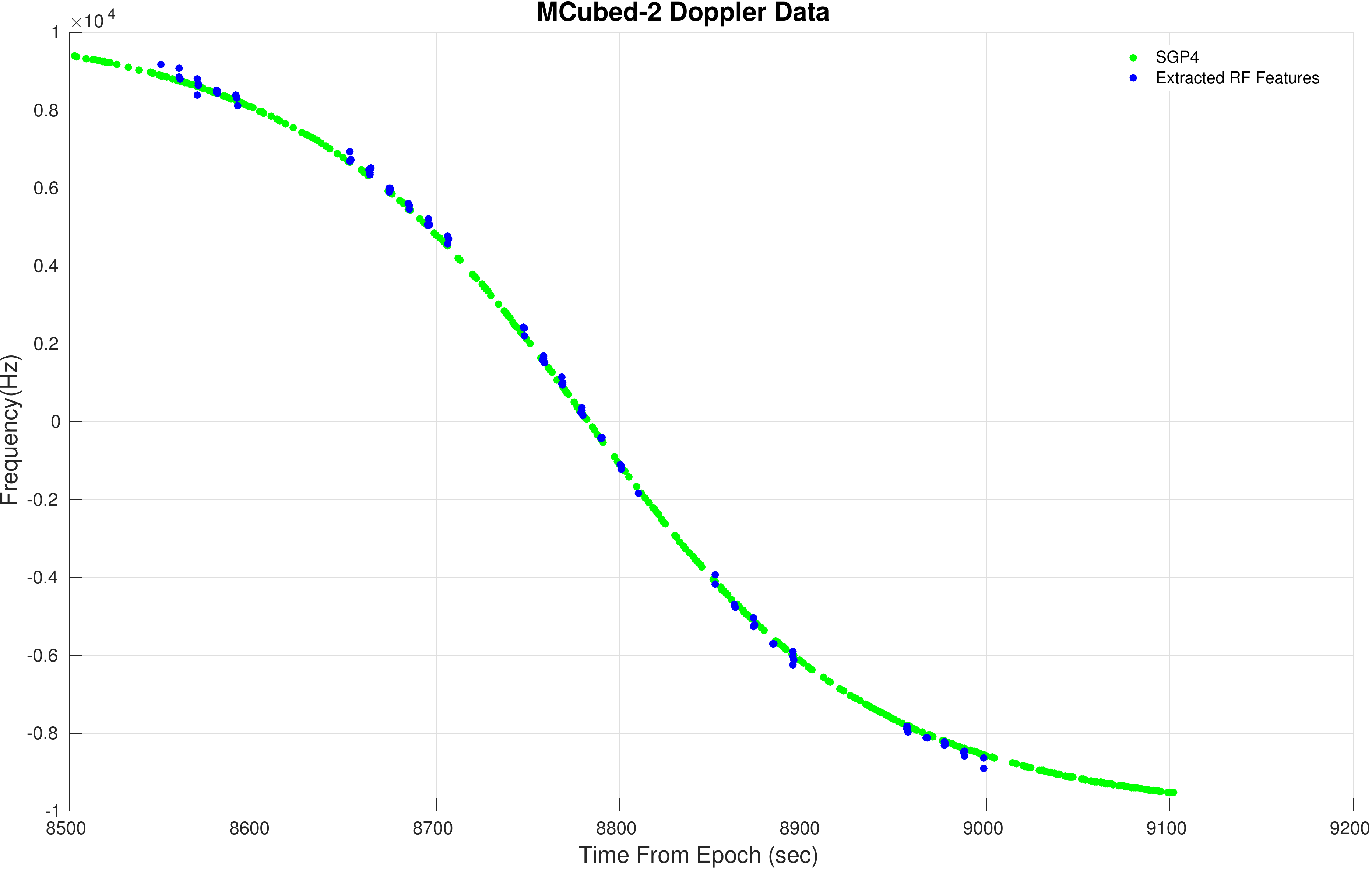}
        \caption{Doppler and Time feature vectors compared with Doppler from TLE Predicted Orbit post bias correction}
        \label{fig:Dopplerfeatures}
    \end{figure}

    The dynamic system for \(U\) and \(V\) used throughout this scenario is the SGP4 propagator. The learning sytems are trained to estimate orbital elements specifically designed for TLE generation (not classical elements) and simplified propagators. \(U\) consists only of Doppler information and \(V\) consists of horizon information for the training data.

    \subsubsection{GRIFEX Orbit Determination}
    \label{ssub:grfxod}
    GRIFEX orbit state was estimated from raw baseband RF transmission data observed over 3 passes and received during nominal operations. The priors were chosen to provide a sufficiently wide region of initial states to test orbit estimation. The prior \(P_\Gamma\) is
    
    \begin{align*}
    A &\sim R_e + U(525,555) \text{ km}, &\qquad
    e &\sim U(0.012, 0.017),\\
    \Omega &\sim U(120^\circ,130^\circ), & \qquad 
    I &\sim U(96^\circ, 101^\circ),\\
    \omega &\sim U(185^\circ, 200^\circ), &\qquad
    M &= U(35^\circ,50^\circ).
    \end{align*}
    
    \(P_\Gamma\) results in a variance in initial position of 765 km. Samples of 4000 orbits were used for training with the two stage sampled regressor (\(~1.35 \times 10^6\) feature points in total). For testing purposes, in addition to the data acquired from on orbit transmissions, additional training data was generated with 200 test orbits for evaluation of the parameters from additional i.i.d samples. The noise distribution \(P(X|F)\) was chosen to be uniform with a width of 200 Hz, similar in behavior to the noise from the Doppler observations. GRIFEX produces beacons and transmissions approximately every 10 seconds with an arbitrary initial offset (depending on operational characteristics) along with spacecraft responses due to nominal operations in between. We approximated the resulting transmission timestamp distribution with a uniform distribution over \(\tilde{\sT}\). 
    
    From raw baseband signals, 534 feature vectors were extracted over 3 passes. The relevant TLE orbital parameters for the JSpOC TLE and the estimated values are shown in Table \ref{table:grifexelements}. The 200 additional simulated test orbits were also tested for orbit determination. The normalized errors in orbital elements for the 200 simulated test orbits for GRIFEX are shown in Figure \ref{fig:grfxtestorbits} (normalized by the width of the support of the prior distribution). The Radial, along-track and cross-track (RSW) and total errors for each of the test orbits are shown in Figure \ref{fig:grfxrswerr}. Orbital elements were estimated for the epoch 01:00:00 2016/2/10 UTC. 

    \begin{table}[h]  
        \caption{Two Line Element Parameters of the GRIFEX spacecraft.}
        \centering
        \begin{tabular}{| l | c | c | c | c | c | c |}
            \hline
            & A(km) & e & I(deg) & \(\Omega\)(deg) & \(\omega\)(deg) & M(deg) \\
            \hline
            True (JSpOC Est.) & 537.663 & 0.0152 & 99.089 & 123.2705 & 194.6996 & 40.8253 \\
            \hline
            Estimated & 534.673 & 0.0167 & 98.43 & 122.709 & 191.4 & 43.7795 \\
            \hline 
        \end{tabular}
        \label{table:grifexelements}
    \end{table}
    The error in estimated initial position for the GRIFEX spacecraft is 30.05 km. The average error for the 200 test orbits was 47.24 km. The error magnitudes of the simulated test orbits are of the same order as that of the experimental data indicating that the fidelity of the training and test models mirror those of experimental data. Note that radial and cross-track errors are significantly lower for Doppler based observations. This is expected as along track information can be gained only through subtle changes in the Doppler curve when working with Doppler based observations and does not change the length of the passes or time between passes. Changes in radial information can be observed as it leads to changes in total variation of the Doppler curves and the timing between passes, resulting significantly better estimates. Changes in cross-track information leads to changes in the length of the passes and total variation of Doppler of the different passes for low Earth orbits.

    \begin{figure} [htbp] 
        \centering
        \includegraphics[width=0.8\linewidth]{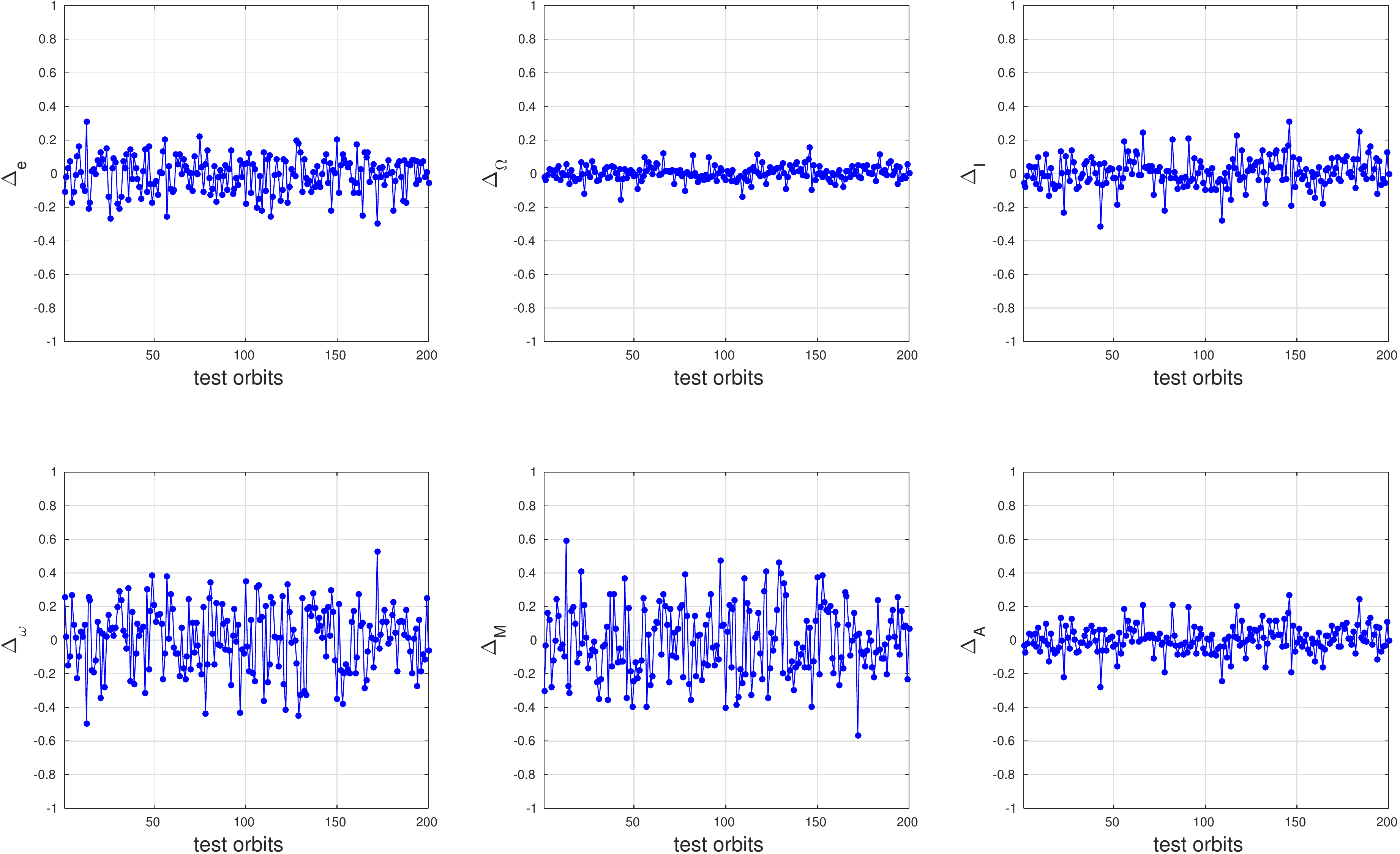}
        \caption{Normalized errors of orbital parameters of test orbits based on GRIFEX Priors}
        \label{fig:grfxtestorbits}
    \end{figure}
    
    \begin{figure} [tb] 
        \centering
        \includegraphics[width=\linewidth]{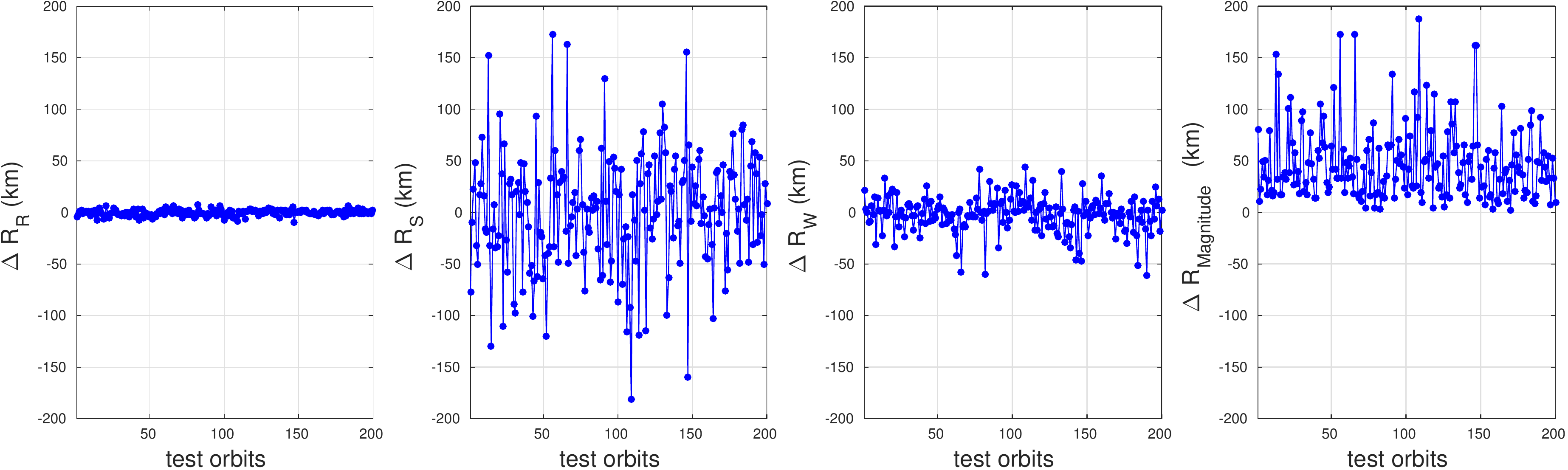}
        \caption{RSW errors of test orbits based on GRIFEX Priors}
        \label{fig:grfxrswerr}
    \end{figure}

    \subsubsection{MCubed-2 Orbit Determination}
    \label{ssub:mc2od}
    MCubed-2 orbit determination was performed with data extracted over 4 passes. The priors for MCubed-2 were chosen to have smaller widths compared to the GRIFEX scenario for variations in RAAN and the argument of perigee to test for changes to estimation behavior while keeping the number of training datapoints approximately equal. The prior \(P_\Gamma\) used is as follows.
    \begin{align*}
        A &\sim R_e + U(635, 665) \text{ km}, &\qquad
        e &\sim U(0.025, 0.03),\\
        \Omega &\sim U(200^\circ, 205^\circ), &\qquad
        I &\sim U(117^\circ, 122^\circ),\\
        \omega &\sim U(65^\circ,70^\circ), & \qquad
        M &= U(223^\circ,233^\circ)
    \end{align*}
    \(P_\Gamma\) results in a variance in initial position of 448 km. The training and testing setups were similar to GRIFEX. Samples of 4000 orbits were used for training the two stage sampled regressor (\(~1.31 \times 10^6\) feature points in total) and 200 additional orbits were sampled for testing. The noise distribution \(P(X|F)\) was chosen to be uniform with a width of 200 Hz. The probability of sampling in time were reduced corresponding to the behavior of MCubed-2.
    
    MCubed-2 produces beacons approximately every 20 seconds with an arbitrary initial offset (depending on operational characteristics). The power levels of these beacons are modulated by the relative orientation of the antennas of the spacecraft and the ground (this can be seen in Figure \ref{fig:rawiq}). We reduce the sampling complexity of this distribution for training data generation by approximating it with a uniform distribution through \(\tilde{\sT}\) which is then selected by the horizon \(O\). A total of 294 feature vectors were extracted over four passes for orbit determination. Table \ref{table:mc2elements} shows the TLE estimated elements versus the estimates from the machine learning algorithm. In addition to the data acquired on orbit, additional data was generated for 200 test orbits for evaluation of the parameters. Figure \ref{fig:mc2testorbits} shows the normalized errors for each orbital element for the test orbits for MCubed-2 (normalized by the width of the support of the prior distribution). Orbital elements were estimated for the epoch: 23:00:00 2016/2/09 UTC.
    
    \begin{table}[h]  
        \caption{Keplerian elements of the MCubed-2 spacecraft.}
        \centering
        \begin{tabular}{| l | c | c | c | c | c | c |}
            \hline
            & A(km) & e & I(deg) & \(\Omega\)(deg) & \(\omega\)(deg) & M(deg) \\
            \hline
            True (JSpOC Estimates) & 644.611 & 0.0273 & 120.493 & 201.978 & 67.501 & 225.47 \\
            \hline
            Estimated & 640.892 & 0.031 & 119.26 & 204.26 & 67.96 & 226.23 \\
            \hline 
        \end{tabular}
        \label{table:mc2elements}
    \end{table}
    The error in estimated initial position for the MCubed-2 spacecraft is 61.91 km. The average error over the test orbits was 22.76 km. The RSW and total errors for the test orbits are shown in Figure \ref{fig:mc2rswerr}. Note the improvement in estimation of the RAAN, inclination, mean anomaly and the semi-major axis and the RSW errors as compared to the estimates in GRIFEX. This is likely due to increased eccentricity, time of observation and decreased total variance in the initial position of the prior. This may also point to increased observability of parameters. Further studies on connections of observability metrics of this system to convergence bounds on learning algorithms should be explored in future work.  
    
    \begin{figure} [htbp] 
        \centering
        \includegraphics[width=0.8\linewidth]{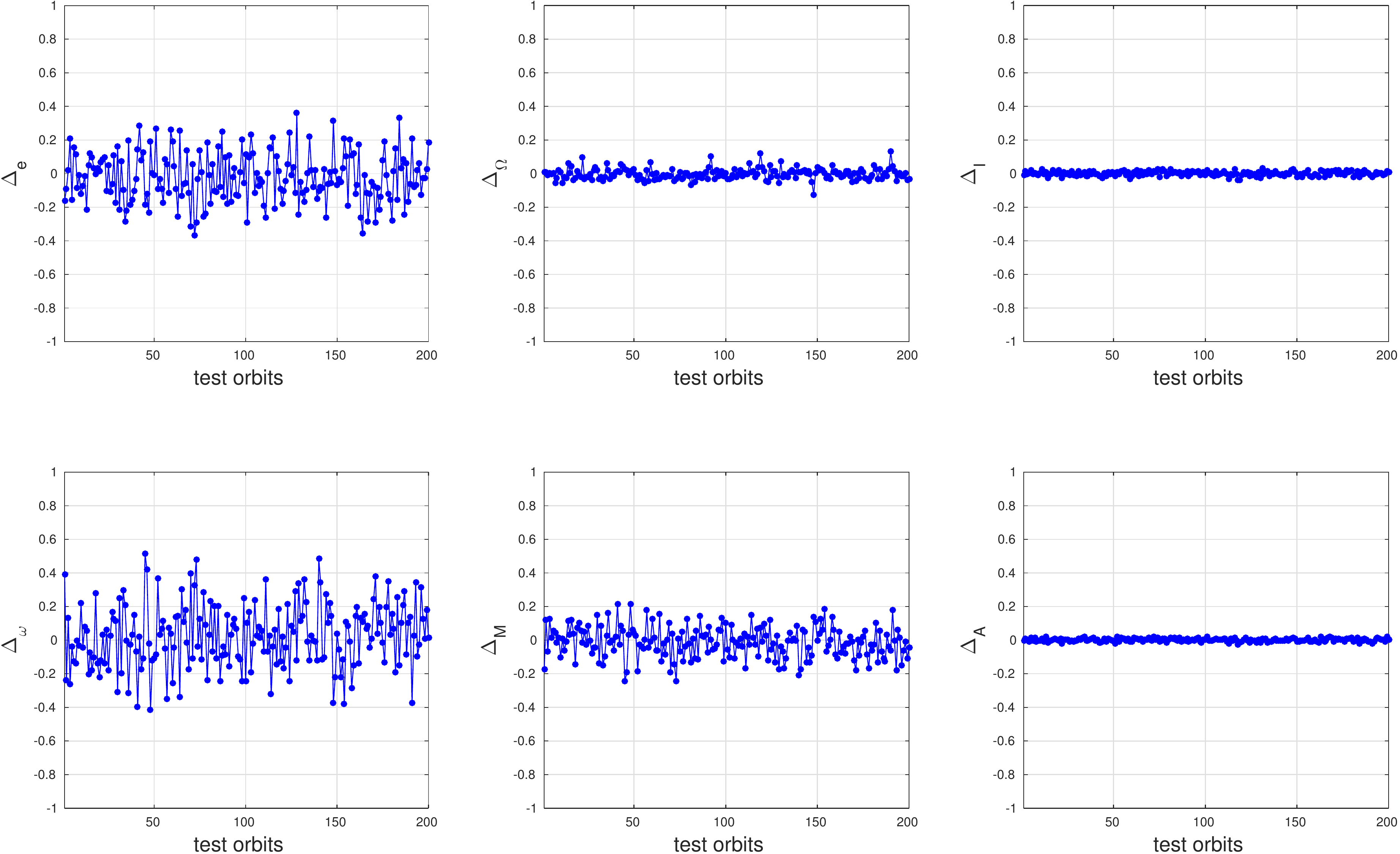}
        \caption{Normalized errors of orbital parameters of test orbits based on MCubed-2 Priors}
        \label{fig:mc2testorbits}
    \end{figure} 

    \begin{figure} [htb] 
        \centering
        \includegraphics[width=\linewidth]{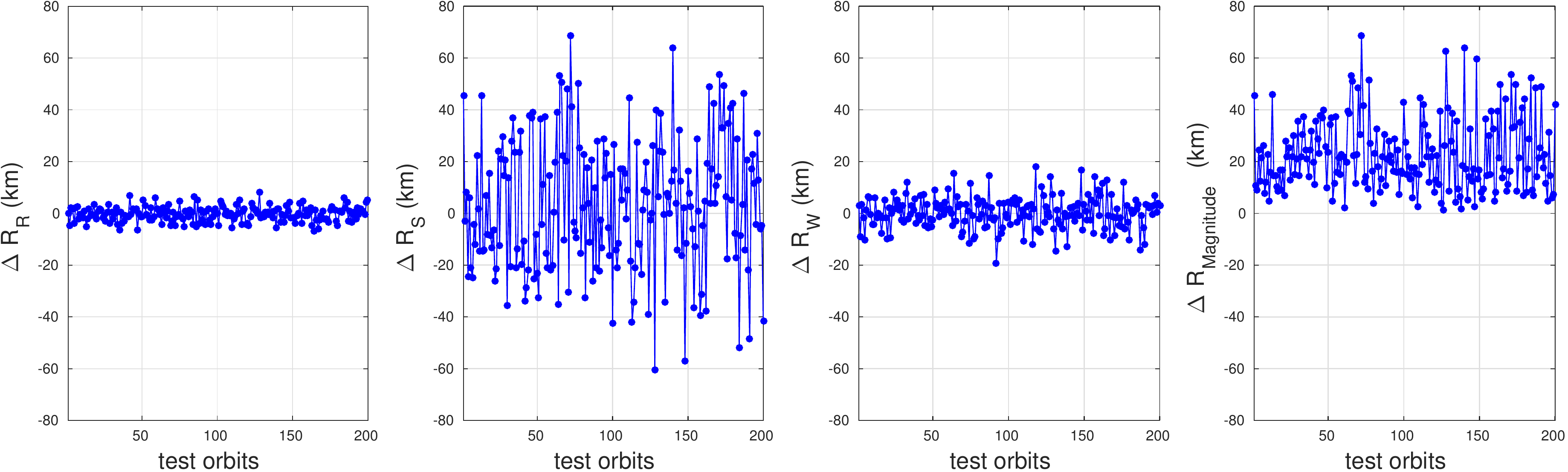}
        \caption{RSW Position Errors of test orbits based on MCubed-2 Priors}
        \label{fig:mc2rswerr}
    \end{figure} 
    
    \subsection{Position Based OD}
    \label{ssub:positionOD}
    We now present the results for a synthetic dataset simulating post deployment orbit determination of two spacecraft using position (direction of arrival and range data) from two ground stations (Ann Arbor and Chicago). This scenario simulates a TLE lottery, and shows that direction of arrival and range features from a noisy RADAR based system can be used to perform identification of the spacecraft based on orbit injection sequence in addition to orbit determination. It also demonstrates indirect generation of the priors of the spacecraft. The algorithm performs both the classification and regression tasks for orbit determination of both spacecraft. 
    The sequence of deployment results in sample information for the classification algorithm to identify the different spacecraft. We shall first describe the details for generation of \(P_\Gamma\) and the propagators used, then describe the learning system and provide results.

    \subsubsection{Sampled Data Generation}
    The priors \(P_\Gamma\) for this scenario are not directly generated and require simulation of deployment scenarios. First, samples are drawn for a deployer spacecraft with the following distribution on orbital parameters:
    
    \begin{align*}
        A &\sim R_e + U(650, 750) \text{ km}, &\qquad
        e &\sim U(0.03, 0.04),\\
        \Omega &\sim U(0^\circ, 5^\circ), &\qquad
        I &\sim U(70^\circ, 75^\circ),\\
        \omega &\sim U(350^\circ,360^\circ), & \qquad
        M &= U(300^\circ,310^\circ)
    \end{align*}

    Two spacecraft are then deployed from the deployer spacecraft. The first spacecraft is provided with a change in velocity (\(\Delta v\)) of -0.5 m/s along the direction of velocity of the deployer. The second spacecraft is inserted 200 s after the first one and is provided with a \(\Delta v\) of +0.5 m/s along the direction of velocity of the deployer to allow the two spacecraft to separate. For both spacecraft, an additional 1.25 m/s is provided in the plane perpendicular to the velocity of the deployer in a direction drawn at random in this plane. The deployments cones of the two spacecraft from the deployer body fixed frame are as shown in Figure \ref{fig:scdeploycone}. The two spacecraft are then allowed to separate by a few km by propagation of their states for 6 hours to simulate passes of multiple small spacecraft whose positions can be resolved by a radar system. The distribution of the two spacecraft states at the 6 hour epoch is \(P_\Gamma\). The training and test distributions are generated the same way.
    
    \begin{figure}[ht] 
        \centering
        \includegraphics[trim={0cm 4cm 0cm 1cm}, width=0.7\linewidth]{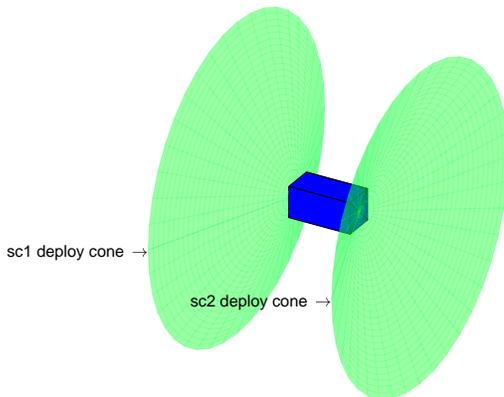}
        \caption{Deployment cones of spacecraft 1 \& 2}
        \label{fig:scdeploycone}
    \end{figure} 

    The analytical propagator used to propagate the two spacecraft worked with the EGM96 gravitational model for spherical harmonics. Coefficients up to 4th order harmonics were taken into consideration along with precession and nutation characteristics, to provide sufficient model fidelity for the synthetic data. The time synchronization errors between the two ground stations were assumed to be negligible resulting in one time-stamp per measurement. The resulting feature vectors were direction of arrival and range information from both ground stations and the time-stamps. This synthetic data generated is similar to those generated by a bi-static radar, and therefore, measurements are generated only when the two spacecraft are in the horizons of both the ground stations. Noise in measurement (\(P_{X|F}\)) of \(0.1^\circ\) was added for azimuth and elevation measurements and 1 km for range measurements at both ground stations, which is generally greater than in practical systems, to test robustness to noise. 
    Data for a total of 4700 orbits were generated. Around 70 feature vectors were generated per orbit per spacecraft for the training datasets. The total variance in initial position was 966 km for each spacecraft. The average separation of the two spacecraft at the epoch for orbit state estimation is 74.92 km.

    \subsubsection{Learning System}
    The first 500 orbit datasets were used along with identifiers for spacecraft to train the transfer learning system. A random Fourier feature approximation based transfer learning approach was applied to improve speed of training the data \citep{blanchard2017domain}. The performance of the algorithm was contrasted against its performance with the traditional SVM classifier in which the data from all the orbits were pooled before classification (pooled classification). The test system consisted of the remaining 4200 orbits whose datapoints had to be classified. 
        \begin{table}[h]
            \caption{Classification error comparison.}
            \centering
            \begin{tabular}{| c | c | c |}
                \hline
                Classification Method & \% Training Error & \% Test Error \\ \hline
                Transfer Learning & 2.01 & 2.57 \\ \hline
                Pooled Classification & 47.32 & 48.62 \\ \hline
            \end{tabular}
        \label{table:transferpooled}
        \end{table}
        Traditional pooled classification systems do not work well for direction of arrival data as the meta-distributions of the two classes can be identical between two different orbit insertion scenarios. However, if the marginal distribution of both spacecraft is known, as it is for transfer learning, the identity of the spacecrafts can be learnt. 

        The output of the classifier is then fed to the regression system for orbit determination of both spacecraft. Note that the classified outputs are used in training to maintain consistency between the training and test distributions as reasoned in section \ref{learningtheory}. Classified points from the first 4000 orbits were used to train each regression system. The orbital elements used were the position and velocity vectors at the epoch instead of traditional Keplerian elements as the argument of perigee and the right ascension angles were no longer compact sets (i.e., it varied as \([x_1, 360)\cup[0,x_2]\), see section \ref{subsec_mathanalysis}), even when the underlying space of probability distributions were compact.  Classified points from 200 orbits were used to test orbital parameter estimates. The average error in estimation of the position of the first spacecraft is 20.06 km. The average error in estimation of the position of the second spacecraft is 19.36 km. Note that this is less than half the average separation of the two spacecraft, so the positions are identifiable and resolvable with information from a single pass. Figures \ref{fig:sc1testorbits} and \ref{fig:sc2testorbits} show the normalized errors of the orbital elements (normalized by the width of support of the prior).

        \begin{figure} [ht] 
            \centering
            \includegraphics[width=0.8\linewidth]{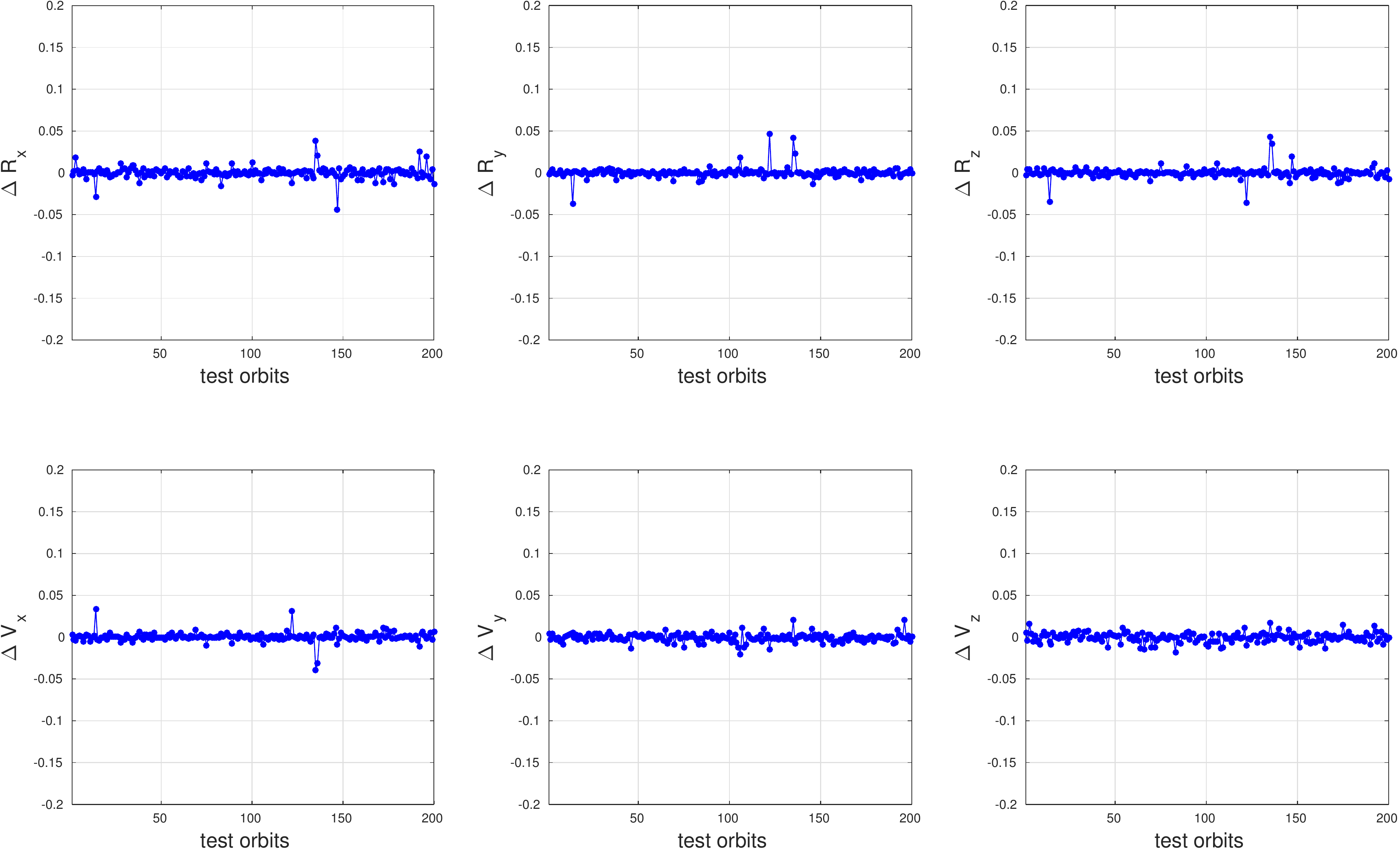}
            \caption{Normalized errors of orbital parameters of test orbits of spacecraft 1}
            \label{fig:sc1testorbits}
        \end{figure} 

        \begin{figure} [ht] 
            \centering
            \includegraphics[width=0.8\linewidth]{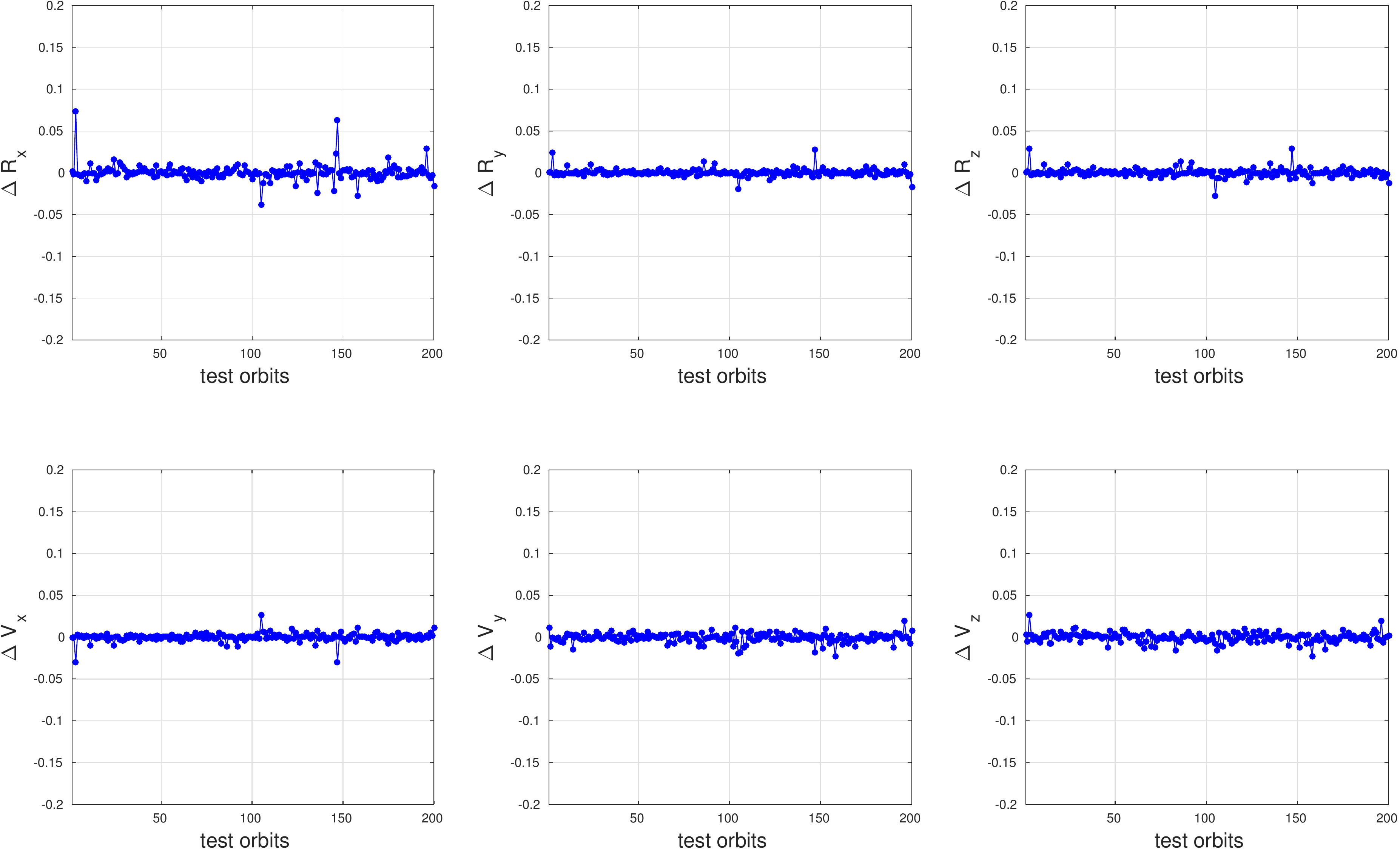}
            \caption{Normalized errors of orbital parameters of test orbits of spacecraft 2}
            \label{fig:sc2testorbits}
        \end{figure} 

    \subsection{Position Based OD - Lunar Orbit}
    \label{ssub:lunarOD}
    The characteristics of the system described in section \ref{sec:mlod} are also satisfied for an \(N\)-body problem. In fact, if the data generation system was constructed with a general celestial dynamical system, no changes will be required to the orbit determination system even with \(N\) bodies. To test the empirical behavior of the algorithms we consider a lunar orbit transfer scenario. The three body problem was the first described example of a chaotic system \citep{musielak2014three, poincare1893methodes, belbruno2004capture} where small changes in the initial orbital parameters lead to large changes in distributions associated with the data. We consider orbit estimation of a 4 day lunar transfer orbit with direction of arrival and range observations from one ground system over one pass. We first describe details of the propagation system and the orbit and prior design and then present the results of the orbit determination scenario.

    \subsubsection{Propagation and Orbit Design}
    \label{ssub:lunorbpropagator}
    For deep space orbit propagation, the propagator used in section \ref{ssub:positionOD} was extended to include accelerations from the Moon, Sun and Jupiter. To simplify and speed up computation, positions of these celestial objects were computed using JPL Ephemerides data. Further computational simplifications were performed by limiting sampling time to 300 seconds over a period of 4 days and interpolating for positions in between (Gaussian splines were used). Epoch for orbit insertion in this hypothetical scenario was chosen to be 18:00:00 1/1/2016 UTC.

    The prior \(P_\Gamma\) was designed as follows. First, a 4 day direct lunar orbit \citep{parker2014low} was designed to obtain a specific trans-lunar injection state. This state was then perturbed in position and velocity from samples drawn from a given set of distributions. An initial orbit using a circular restricted 3 body problem was constructed in a synodic frame. The synodic frame was transfered to a 3 dimensional system with the appropriate transformation. The circular lunar orbit was then replaced with JPL Ephemerides and the initial states were perturbed to obtain a transfer orbit. Lunar spherical harmonic coefficients were not taken into consideration and the moon was treated as a sphere due to negligible perturbatory effects during the test period. The orbits designed were similar to the 4 day injection orbits described in \citet{parker2014low}. The initial state of this orbit was used as initial input parameters to the distribution \(P_\Gamma\). The perturbations for position and velocity were designed as compact sets of conic sections with the following distributions (States are in spherical coordinates):
    \begin{align*}
        R_r &\sim R_{r,init} + U(0, 0.05R_{r,init}) \text{ km}, &\qquad
        V_r &\sim V_{r,init} + U(-0.02V_{r,init}, 0.02V_{r,init}), \\
        R_\theta &\sim U(R_{\theta, init} - 2^\circ, R_{\theta, init} + 2^\circ), &\qquad
        V_\theta &\sim U(V_{\theta, init} - 1^\circ, V_{\theta, init} + 1^\circ), \\
        R_\phi &\sim U(R_{\phi, init} - 2^\circ , R_{\phi, init} + 2^\circ), &\qquad
        V_\phi &\sim U(V_{\phi, init} - 1^\circ , V_{\phi, init} + 1^\circ).
    \end{align*}

    This prior results in an effective variance of 203 km in initial position. The samples drawn from the above distributions are then used to evaluate the orbit determination system. Figure \ref{fig:lunarorbitplot} shows the paths generated by the analytical propagator in the ECI frame for 20 sample states drawn from \(P_\Gamma\). A uniform distribution over time was used to generate observation vectors from one ground station. The number of samples were chosen to result in approximately one sample every 3 minutes over one pass (<12 hour period). Note that this 12 hour period does not begin during orbit insertion but 6 hours after insertion (This is due to the fact that the spacecraft is not in view of the chosen ground station during orbit insertion. Besides interpolation errors generated by the propagator, no additional noise was added to the system to generate observations. 
    
        \begin{figure} [ht] 
            \centering
            \includegraphics[width=0.9\linewidth]{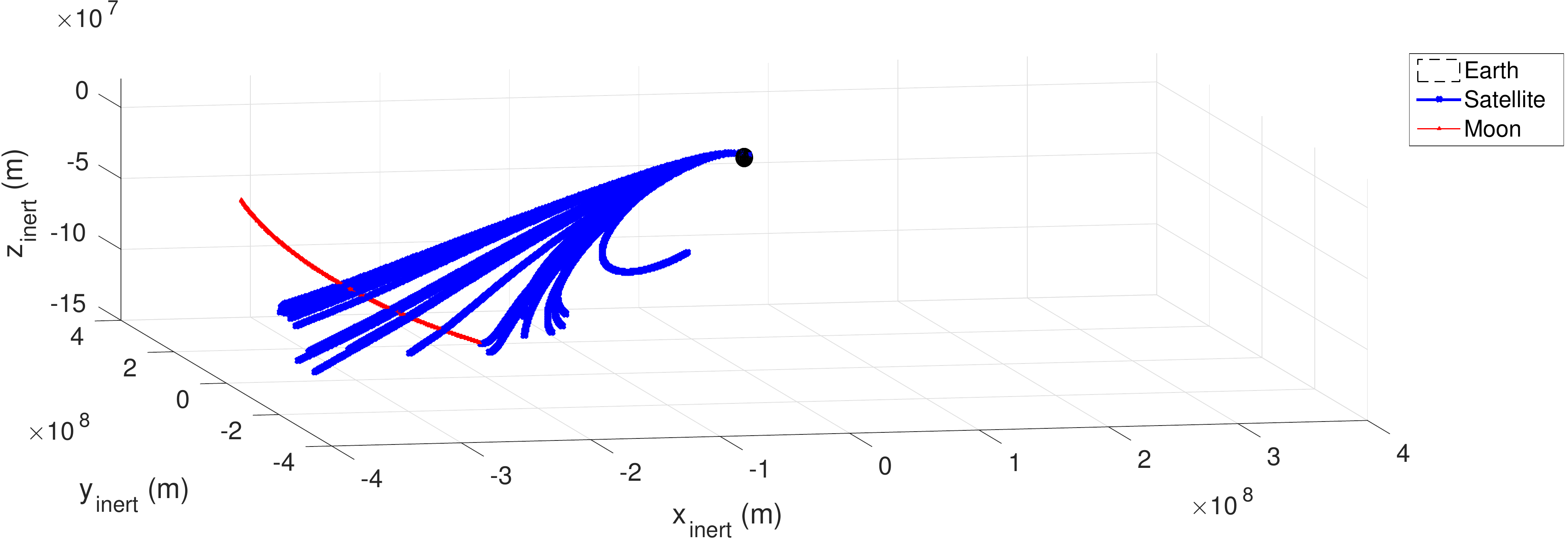}
            \caption{Example lunar transfer orbits drawn from prior distribution}
            \label{fig:lunarorbitplot}
        \end{figure} 

        \begin{figure} [ht] 
            \centering
            \includegraphics[width=0.8\linewidth]{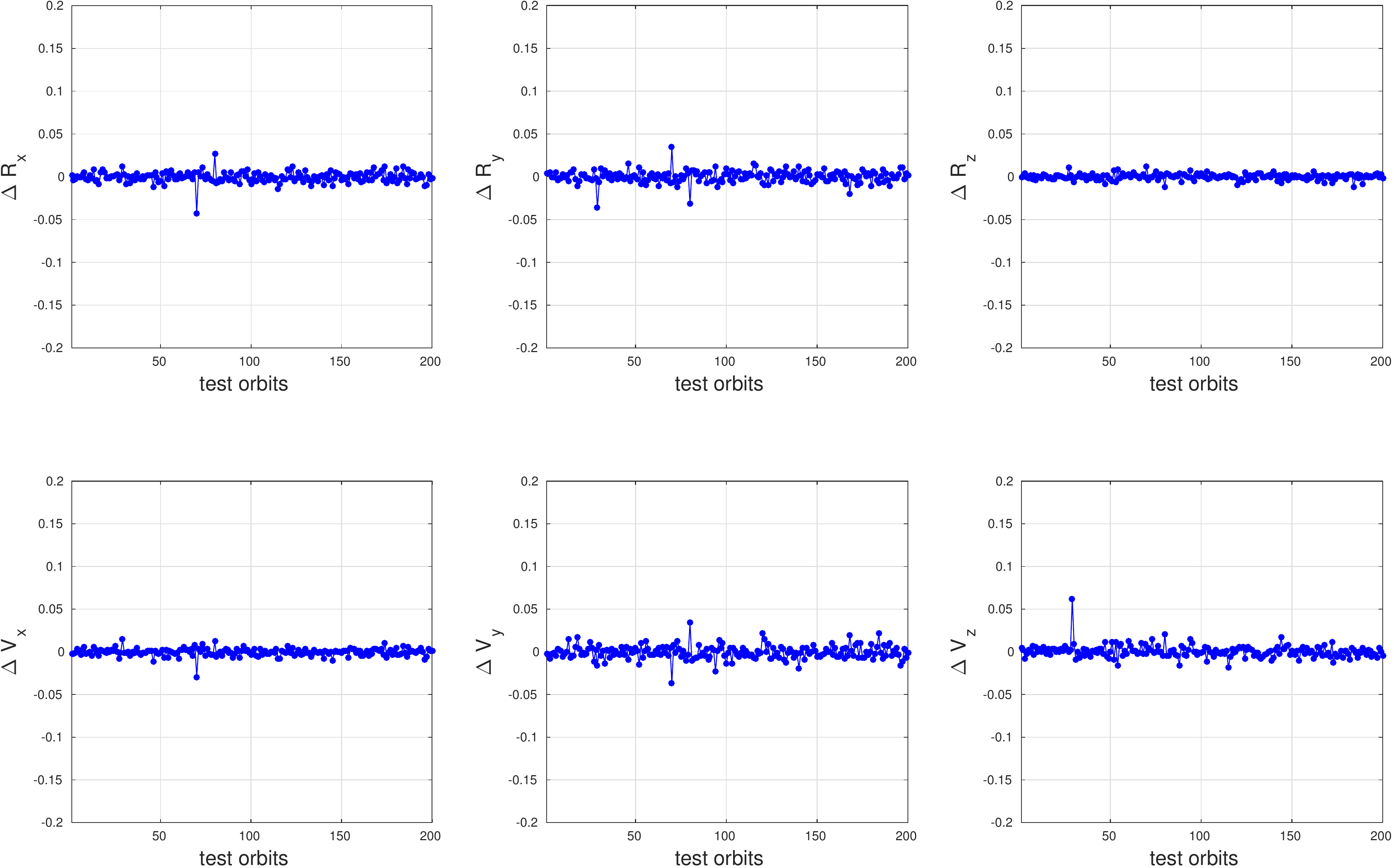}
            \caption{Normalized errors of orbital parameters of test orbits for trans-lunar scenario}
            \label{fig:lunartestorbits}
        \end{figure} 

    \subsubsection{Learning System}
    \label{ssub:learningsys_lunarOD}
    Similar to the other scenarios, 4000 orbits were generated for training and 200 for testing. The preprocessing step was modified to normalize time period and range variations, besides this no changes were performed to the learning algorithm. Figure \ref{fig:lunartestorbits} shows the normalized errors in the estimates (normalized by the width of the support of the prior of each element). The average error in position estimation was 4km. This error is lower in comparison to similar scenarios which use direction and range information for orbit estimation. Note that due to finite sampling and chaotic nature of the orbits, outliers will exist with very low probability. 

    \subsection{Position based OD - Comparison with EKF}
    We compare the performance of the learning technique proposed with a traditional orbit determination system based on the EKF. The propagators used were the same as described in section \ref{ssub:positionOD}. We shall first describe the exact characteristics of the orbit determination system used to compare against and then provide details of synthetic data generated for comparison.
    
    \subsubsection{Sample Data Generation}
    \label{ssub:sampledataEKFOD}
    The probability distribution over orbital parameters chosen for this scenario was designed from two perspectives. First the preliminary orbit determination system's performance with large noise added to the observations should be sufficient to force the EKF out of its linear region. Second, the noise added should be admitted by the data preprocessing and \(6 \sigma\) editing filters used in state of the art EKF based orbit determination algorithms \citep{wright2013odtk}. Orbits with relatively high eccentricity were chosen with the following priors:

    \begin{align*}
        A &\sim R_e + U(5000, 5400) \text{ km}, &\qquad
        e &\sim U(0.4, 0.35),\\
        \Omega &\sim U(350^\circ, 5^\circ), &\qquad
        I &\sim U(70^\circ, 75^\circ),\\
        \omega &\sim U(0^\circ,10^\circ), & \qquad
        M &= U(300^\circ,320^\circ)
    \end{align*}

    The propagator described in section \ref{ssub:positionOD} was used to generate data over one pass. Uniform noise with width of \(0.2^{\circ}, 0.2^{\circ}, 2 \text{ km}\) is added to the azimuth, elevation and range measurements respectively. A total of 4200 random orbits were generated.
    
    \subsubsection{EKF Based OD}
    \label{ssub:EKFOD}
    This orbit determination system consisted of a preliminary orbit determination system for initialization followed by the EKF. The preliminary orbit determination system used was Herrick-Gibbs \citep{vallado2001fundamentals}. The preliminary orbit determination was conducted on points on a section of the orbit near the perigee and the points were chosen such that the time period between the points was about 10 minutes, based on the results of the performance with the ascending Molniya scenarios in \citet{schaeperkoetter2011comprehensive}. The preliminary OD system was followed by an EKF with \(6 \sigma\) data editing (see \citet{wright2013odtk}). The dynamic system used for propagation of the EKF is identical to the propagator used for generation of the observations. This was done to compare the performance of the EKF in scenarios with significantly noisy observations. 
    
    For a parity in comparison of the two techniques the same set of orbits were used for parameter selection for the EKF and training of the machine learning based OD system. The first 4000 orbits were used to generate the error covariance matrices for the EKF. The 4000 orbits were also used to train the machine learning algorithm (5 fold cross validation). No changes to the learning algorithm were made from the previous sections. Both the EKF and the learning based OD system were tested on the data points generated from the last 200 orbits. Figure \ref{fig:ekfcomparison} shows the initial position errors for the 200 orbits under test. Note that if the EKF diverges to a point where no observations lie in the \(6 \sigma\) range the measurement editor will edit out all further observations limiting further updates. As can be seen, the learning based orbit determination system has significant performance advantage over the EKF, albeit under significantly larger computational requirements. The few outliers for the learning based system will converge to zero in probability with increase in training data as detailed in section \ref{learningtheory}.
    \begin{figure} [ht] 
        \centering
        \includegraphics[width=0.9\linewidth]{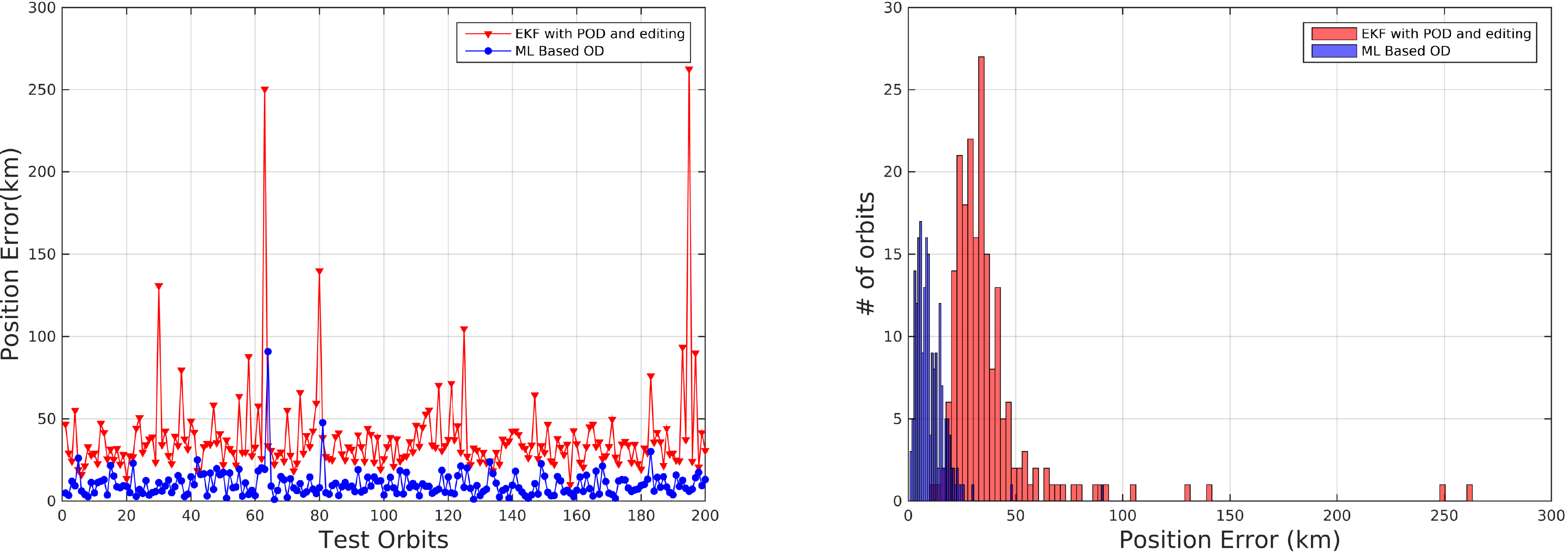}
        \caption{Comparison of EKF and Learning Based OD}
        \label{fig:ekfcomparison}
    \end{figure}

    \subsection{Discussion}
    \label{sub:discussion}
    A summary of the position error results is shown in Table \ref{table:AllTestErrors}. For Doppler based orbit determination, the Along-track and Cross-track errors are larger. This is a direct consequence of the fact that each individual point contains very little actual position information and position can only be gained from the changes in the probability distribution that generate the points. Radial Errors as low as radial information can be gained from zero Doppler cross-over points. The synthetic test data error magnitudes are of the same order as the errors produced by the experimental datasets. If the data generation systems are not sufficiently realistic, there can be descripancies in the test errors and the errors produced by the experimental datasets, as the learnt system will not directly correspond to the experimental data. This requirement also applies to the noise modeling of the Doppler measurements. While the MC2 orbit determination position errors are larger in comparison to GRIFEX, the equivalent comparison in terms of the orbital elements themselves produces the opposite result. This is a consequence of the fact that the optimization to compute orbital elements does not directly correspond to reducing position errors as the transformation between the two is not linear.

    With a constant number of training orbits (4000), decreasing uncertainty and improving observability improves accuracy. This behaviour can be seen in two scenarios. The accuracy of Doppler-only OD is lower in comparison to Position based orbit determination due to differences in observability and noise effects. The accuracy is highest for chaotic orbits with small initial spaces, where small changes in the initial condition produce very large changes in the orbit. Note here that while the lunar orbit scenario observation intervals were for 10 hours, the average transmissions characteristics produced equivalently lower number of transmissions per orbit such that the datasets of the position based orbit determination systems had the same order of training and test datapoints per orbit as in the LEO case.
    
    For position based orbit determination, the along-track errors are larger. This behaviour is expected as velocity information cannnot be directly gained from the features. The sum of average and RMS errors for the two satellites is less than the average separation between the satellites, and the spacecraft can be resolved in the orbit insertion scenario. Note that for the position based OD scenarios, noise of \( (0.1^\circ, 0.1^\circ, 1 \text{km}) \) were added to the (Azimuth, Elevation, Range) measurements.

        \begin{table}[h]
            \centering
            \caption{Comparison of test errors}
            \begin{tabular}{@{}*{7}{ c }@{}}
                \toprule
                & \thead{Prior Position \\ Std. Dev (km)} 
                & \thead{\(|\tilde{\sT}|\) \\ (hr)}
                & \thead{Radial Error \\ (km)}
                & \thead{Along-track \\ Error (km)}
                & \thead{Cross-track \\ Error (km)}
                & \thead{Total Error \\ (km)} \\

                \midrule
                \thead{GRIFEX \\ Synthetic} & 765 & 4.5 & 2.9581 & 56.5984 & 17.4887 & 59.3126 \\
                \thead{GRIFEX \\ Experimental} & 765 & 4.5  & 6.92 & 7.33 & -28.3050 & 30.05 \\
                \thead{MC2 \\ Synthetic} & 448 & 7  & 2.859 & 25.8533 & 6.1914 & 26.73 \\
                \thead{MC2 \\ Experimental} & 448 & 7 & 13.4617 & 28.9652 & -53.0385 & 61.91 \\
                \thead{Position \\ Synthetic (LEO) 1} & 966 & 1 & 15.7595 & 38.8545 & 10.7878  & 43.29  \\
                \thead{Position \\ Synthetic (LEO) 2} & 966 & 1 & 10.5967 & 25.3397 & 2.7114 & 27.5996 \\
                \thead{Positon \\ Synthetic (Lunar)} & 203 & 10 & 1.3158 & 5.2577 & 1.4032 & 5.5986 \\
                \bottomrule
            \end{tabular}
            \label{table:AllTestErrors}
        \end{table}

    \section{Conclusions and Future Work} \label{sec:conclusions}
    We presented the orbit determination problem of multiple spacecraft from a learning theoretic perspective. The learning system allows for estimation of spacecraft orbits over a very broad set of conditions. The learning algorithm requires only bounded and compact space specifications without the need for initialization in the linear region of the estimator, unlike traditional non-linear estimators. We showed that the combined algorithm is consistent when the mapping is continuous and the classifiers are well defined. We presented experimental results for Doppler-only orbit determination scenarios with operational spacecraft and synthetic deep space orbit scenarios. We also provide comparisons with the EKF in a synthetic scenario with large measurement noise, where the proposed approach overcomes the divergence limitations of the EKF. The learning approach can also be used to perform state estimation in weakly observable, unactuated dynamic systems with random and noisy observations over finite time periods.

    Future work will provide extensions and implementation architectures to integrate the proposed approach to federated ground station networks with low-gain antenna systems, advancing autonomy in spacecraft operations and orbit monitoring, supplementing existing orbit determination capabilities. We also intend to provide approaches to integrate multiple independent sources of observations with different types of sensors to perform orbit determination.
    
    \section*{Acknowledgments}
    The research described in this publication was carried out with the support of the Strategic University Research Partnership (SURP), 2015, under contract with the Jet Propulsion Laboratory, California Institute of Technology, under a contract with the National Aeronautics and Space Administration (NASA). We also thank Jon Hamkins and Sam Dolinar for their support and encouragement in this effort.

    \small{
        \bibliographystyle{ieeetr}
        \bibliography{Mainref.bib}
    }
    
    \section*{Appendix}
    \renewcommand{\thesubsection}{\Alph{subsection}}
    
    \subsection{Proof for theorem \ref{th:lt1}} \label{proof:th:lt1}
    \begin{proof}
        Let \(\sZ = \phi(\sB_\sX) \times \tsJ\). Let us denote by \(l_{01}\) the Bayes loss. We have
        \begin{equation*}
            \begin{aligned}
                \n{(A_{\hat{h}} - A_{h^\ast})r}_\rho^2 &= \int_\sZ \n{(\ck_{\hat{h}(\phi_p)}^\ast - \ck_{h^\ast(\phi_p)}^\ast) r}_\tsJ^2 d\rho(\phi_p,\gamma) \\
                & \leq \n{r}_{\sH_{\bar{k}}}^2 \int_{\sZ} \n{(\ck_{\hat{h}(\phi_p)}^\ast - \ck_{h^\ast(\phi_p)}^\ast)}_{\sL(\sH_\ck,\tsJ)}^2 d\rho \\
                & = \n{r}_{\sH_{\bar{k}}}^2 \int_{\sZ} \n{(\ck_{\hat{h}(\phi_p)} - \ck_{h^\ast(\phi_p)})^\ast}_{\sL(\sH_\ck,\tsJ)}^2 d\rho.
            \end{aligned}
        \end{equation*}
        Using \(\n{\cdot}_\sL \leq \n{\cdot}_{\sL_2}\) and the H\"{o}lder condition,
        \begin{equation*}
            \begin{aligned}
                \n{(A_{\hat{h}} - A_{h^\ast})r}_\rho^2 & \leq \n{r}_{\sH_{\bar{k}}}^2 \int_{\sZ} \n{(\ck_{\hat{h}(\phi_p)} - \ck_{h^\ast(\phi_p)})}_{\sL_2(\tsJ,\sH_\ck)}^2 d\rho \\
                & \leq \n{r}_{\sH_{\bar{k}}}^2 L_\ck \int_{\phi(\sB_\sX)} \n{\hat{h}(\phi_P) - h^\ast(\phi_P)}_{\sH_{\bar{k}}}^{2\beta} d\rho(\phi_P) \\
                & = \n{r}_{\sH_{\bar{k}}}^2 L_\ck \int_{\sB_\sX} \nbig{ \int_\sX \bar{k}(x,\cdot) \onev_{\set{\sgn(\hat{g}(\phi_p,x)) = i}} dP_{X} - \int_\sX \bar{k}(x,\cdot) \onev_{\set{\sgn(g^\ast(\phi_p,x)) = i}} dP_{X}}_{\sH_{\bar{k}}}^{2\beta} d\rho_P(P_X) \\
                & = \n{r}_{\sH_{\bar{k}}}^2 L_\ck \int_{\sB_\sX} \nbig{ \int_\sX \bar{k}(x,\cdot) \BgB{ \onev_{\set{\sgn(\hat{g}(\phi_p,x)) = i}} - \onev_{\set{\sgn(g^\ast(\phi_p,x)) = i}} } dP_{X}}_{\sH_{\bar{k}}}^{2\beta} d\rho_P(P_X). 
            \end{aligned}
        \end{equation*}
        From convexity of \(\n{\cdot}_{\sH_{\bar{k}}}\) and Jensen's inequality,
        \begin{equation*}
            \begin{aligned}
                \n{(A_{\hat{h}} - A_{h^\ast})r}_\rho^2 & \leq \n{r}_{\sH_{\bar{k}}}^2 L_\ck \int_{\sB_\sX} \BggB{ \int_\sX \n{\bar{k}(x,\cdot)}_{\sH_{\bar{k}}} \big| \onev_{\set{\sgn(\hat{g}(\phi_p,x)) = i}} - \onev_{\set{\sgn(g^\ast(\phi_p,x)) = i}} \big| dP_{X}}^{2\beta} d\rho_P(P_X) \\
                & \leq \n{r}_{\sH_{\bar{k}}}^2 L_\ck B_{\bar{k}}^{2\beta} \int_{\sB_\sX} \BggB{ \int_\sX \big| \onev_{\set{\sgn(\hat{g}(\phi_p,x)) = i}} - \onev_{\set{\sgn(g^\ast(\phi_p,x)) = i}} \big| dP_{X}}^{2\beta} d\rho_P(P_X) \\
                & \leq \n{r}_{\sH_{\bar{k}}}^2 L_\ck B_{\bar{k}}^{2\beta} \int_{\sB_\sX} \BggB{ \int_\sX \onev_{\set{\sgn(\hat{g}(\phi_p,x)) \neq \sgn{g^\ast}(\phi_p, x}} dP_{X}}^{2\beta} d\rho_P(P_X) \\
                & \leq \n{r}_{\sH_{\bar{k}}}^2 L_\ck B_{\bar{k}}^{2\beta} \int_{\sB_\sX} H^{2\beta} d\rho_P(P_X)
            \end{aligned}
        \end{equation*}
        We will first provide some observations and simplify \(H\). Let \(\eta(X) = P(Y=i | X=x) \). Fix \(\epsilon > 0\). We can say the following about \(H\):
        \begin{equation*}
            \begin{aligned}
                H & \leq \frac{1}{\epsilon} \int_\sX \big| \eta(X)-\frac{1}{2} \big| \onev_{\set{\sgn(\hat{g}(\phi_p,x)) \neq \sgn{g^\ast}(\phi_p, x)}} dP_{X} 
                + \int_{X:|\eta(X)-0.5|<\epsilon} \onev_{\set{\sgn(\hat{g}(\phi_p,x)) \neq \sgn{g^\ast}(\phi_p, x)}} dP_X \\
                & \leq \frac{1}{\epsilon} \int_\sX \big| \eta(X)-\frac{1}{2} \big| \onev_{\set{\sgn(\hat{g}(\phi_p,x)) \neq \sgn{g^\ast}(\phi_p, x)}} dP_{X}
                        + \int_X \onev_{\set{|\eta(X)-0.5|<\epsilon}} dP_X \\
                  & = \frac{1}{2 \epsilon} (\sR(\hat{g}, P_{XY}, l_{01}) - \sR(g^\ast, P_{XY}, l_{01})) + P_X(\set{|\eta(X)-0.5|<\epsilon}) \\
                  & = H_1 + H_2
            \end{aligned}
        \end{equation*} 
        Also,
        \begin{equation*}
                H^{2\beta} = (H_1 + H_2)^{2\beta} \leq C_{2\beta}(H_1^{2\beta} + H_2^{2\beta})
        \end{equation*}
        Where \(C_{2\beta} = 1\) when \(\beta \in (0,0.5)\) (from subadditivity property), \(C_{2\beta}=1\) when \(\beta = 0.5\) (all inner terms are positive) and \(C_{2\beta} = 2^{2\beta}\) when \(\beta \in (0.5,1]\). The argument for \(\beta \in (0.5,1]\) is as follows:

        \begin{equation*}
            (H_1 + H_2)^{2\beta} \leq (2 \max \set{H_1, H_2})^{2\beta} \leq 2^{2\beta} (\max \set{H_1, H_2})^{2\beta} \leq 2^{2\beta} (H_1^{2\beta} + H_2^{2\beta})
        \end{equation*}
        Simplifying,
        \begin{equation*}
            H^{2\beta} \leq \frac{1}{\epsilon^{2\beta}} (\sR(\hat{g}, P_{XY}, l_{01}) - \sR(g^\ast, P_{XY}, l_{01}))^{2\beta} + (P_X(\set{|\eta(X)-0.5|<\epsilon}))^{2\beta}
        \end{equation*}

        Therefore we have 
        \begin{align*}
            \n{(A_{\hat{h}} - A_{h^\ast})r}_\rho^2
            & \leq \begin{aligned} \n{r}_{\sH_{\bar{k}}}^2 L_\ck B_{\bar{k}} 
                    & \bigg[ \BgB{\frac{1}{\epsilon}}^{2\beta} \E_{P_{XY}} \big[ (\sR(\hat{g}, P_{XY}, l_{01}) - \sR(g^\ast, P_{XY}, l_{01}))^{2\beta} \big] \\
                    & + \E_{P_X} \big[ P_X(\set{|\eta(X)-0.5|<\epsilon})^{2\beta} \big] \bigg]\end{aligned} \\
            & \leq \begin{aligned} \n{r}_{\sH_{\bar{k}}}^2 L_\ck B_{\bar{k}} 
                    & \bigg[ \BgB{\frac{1}{\epsilon}}^{2\beta} \E_{P_{XY}} \big[\sR(\hat{g}, P_{XY}, l_{01}) - \sR(g^\ast, P_{XY}, l_{01}) \big]^{\max \set{1, 2\beta}} \\
                    & + \E_{P_X \sim \rho, X \sim P_X} \big[\onev_{\set{|\eta(X)-0.5|<\epsilon}}\big]^{\max \set{1, 2\beta}} \bigg] \end{aligned}
        \end{align*}

        where \(\sR\) is the Bayes error of the classifier for a given distribution \(p\). The last inequality follows from Jensen's inequality when \(0 < \beta \leq 0.5\) and from the fact that \((\sR(\hat{g}(\phi_p)) - \sR^\ast)^{2\beta} < (\sR(\hat{g}(\phi_p)) - \sR^\ast)\) and \((P_X(\set{|\eta(X)-0.5|<\epsilon}))^{2\beta} \leq P_X(\set{|\eta(X)-0.5|<\epsilon})\) when \(0.5 < \beta \leq 1\). Let \(P_X(\set{|\eta(X)-0.5|<\epsilon}) = \tilde{\epsilon}\). Switching to the surrogate loss with classification calibration and calibration function \(\varphi\) \citep{bartlett2006convexity}
        \begin{equation*}
        \begin{aligned}
            \n{(A_{\hat{h}} - A_{h^\ast})r}_\rho^2 & \leq \n{r}_{\sH_{\bar{k}}}^2 L_\ck B_{\bar{k}}^{2\beta} \bigg[ \BgB{\frac{1}{\epsilon}}^{2\beta} \big[\varphi^{-1} \BgB{(I(\hat{g},\infty) - I(g^\ast,\infty)} \big]^{\min(1,2\beta)} + \tilde{\epsilon}^{\min(1,2\beta)}\bigg]
        \end{aligned}
        \end{equation*}
        Next, we shall use universal consistency arguments presented for corollary 5.4 in \citet{blanchard2011generalizing} (see \citet{blanchardsupplementary} for arguments) and evaluate for \(\n{r}_{\sH_{\bar{k}}} = 1\). Assume that the number of datasets for the marginal predictor \(N_{trans}\) and the number of training points per dataset generated \(n_{trans}\) are such that \(N_{trans} = \mathcal{O}(n_{trans}^q)\) for some \(q > 0\). Then, if there is a sequence of \(\epsilon_j\) and a sequence of regularization parameter \(\xi_{1,j}\) is such that \(\epsilon_j \rightarrow 0, \xi_{1,j} \rightarrow 0\) and \(\xi_{1,j} \epsilon_j^{2\beta} \sqrt{\frac{j}{\log j}} \rightarrow \infty \)
        \begin{equation*}
            \n{(A_{\hat{h}} - A_{h^\ast})} \leq \lim_{\epsilon_j \rightarrow 0} \tilde{\epsilon}^{\min(0.5,\beta)}
        \end{equation*}
        in probability. From the above equation, if 
        \begin{equation*}
            \lim_{\epsilon_j \rightarrow c} \E_{P_X \sim \rho, X \sim P_X} \big[\onev_{\set{|\eta(X)-0.5| < \epsilon_j}}\big] = 0
        \end{equation*}
        then
        \begin{equation*}
            \n{(A_{\hat{h}} - A_{h^\ast})} \rightarrow 0
        \end{equation*}
        in probability. The question now remains whether it is always possible to construct such a sequence. When \(c > 0\), the solution is trivial and is a straightforward extension of the sequence that can be constructed for universal consistency in \citet{blanchard2011generalizing}. When \(c=0\), we first construct a sequence \(\xi^\prime_{1,j}\) which converges when \(I(\hat{g},\infty) \rightarrow \inf I(g,\infty)\) and split the sequence as \(\xi_{1,j} = (\xi_{1,j}^\prime )^{0.5}\) and \(\epsilon_j = (\xi_{1,j}^\prime)^{1/(4\beta)}\).



    \end{proof}

    \subsection{Proof for Lemma \ref{lem:lt2}} \label{proof:lem:lt2}

    \begin{proof}
        The proof is similar to that of theorem \ref{th:lt1}. The training dataset consists of points drawn from probability distributions \(\{P_i\}_{i=1}^{N_{dr}}\) resulting in embeddings \(\{\phi_i\}_{i=1}^{N_{dr}}\). We have 
        \begin{equation*}
            \begin{aligned}
                \n{(\hat{T}_{\hat{h}} - \hat{T}_{h^\ast})r}^2
                & \leq \frac{N_{dr}}{N^2_{dr}} \sum_{i=1}^{N_{dr}} \n{(\ck_{\hat{h}(\phi_i)}\ck_{\hat{h}(\phi_i)}^\ast - \ck_{h^\ast(\phi_i)}\ck_{h^\ast(\phi_i)}^\ast)r}^2 \\
                & \leq \frac{ \n{r}^2 }{N_{dr}} \sum_{i=1}^{N_{dr}} \n{(\ck_{\hat{h}(\phi_i)}\ck_{\hat{h}(\phi_i)}^\ast
                    - \ck_{\hat{h}(\phi_i)}\ck_{h^\ast(\phi_i)}^\ast
                    + \ck_{\hat{h}(\phi_i)}\ck_{h^\ast(\phi_i)}^\ast
                    - \ck_{h^\ast(\phi_i)}\ck_{h^\ast(\phi_i)}^\ast)}^2 \\
                & \leq \frac{ 2 \n{r}^2 }{N_{dr}} \sum_{i=1}^{N_{dr}} \n{ (\ck_{\hat{h}(\phi_i)}\ck_{\hat{h}(\phi_i)}^\ast
                    - \ck_{\hat{h}(\phi_i)}\ck_{h^\ast(\phi_i)}^\ast }^2
                    + \n{ \ck_{\hat{h}(\phi_i)}\ck_{h^\ast(\phi_i)}^\ast
                    - \ck_{h^\ast(\phi_i)}\ck_{h^\ast(\phi_i)}^\ast) }^2 \\
                & \leq \frac{ 4 \n{r}^2 B_\ck^2 }{N_{dr}} \sum_{i=1}^{N_{dr}} \n{ \ck_{\hat{h}(\phi_i)} - \ck_{h^\ast(\phi_i)} }^2 \\
                & \leq \frac{ 4 \n{r}^2 B_\ck^2 L_\ck }{N_{dr}} \sum_{i=1}^{N_{dr}} \n{\hat{h}(\phi_i) - h^\ast(\phi_i)}^{2\beta}
            \end{aligned}
        \end{equation*}
        Following similar arguments as those presented for theorem \ref{th:lt1} we have 
        \begin{equation*}
            \begin{aligned}
                \n{(\hat{T}_{\hat{h}} - \hat{T}_{h^\ast})r}^2
                & \leq \frac{ 4 \n{r}^2 B_\ck^2 L_\ck }{N_{dr} \epsilon^{2\beta}} \sum_{i=1}^{N_{dr}} \big[ (\sR(\hat{g}, P_{i, XY}, l_{01}) - \sR(g^\ast, P_{XY}, l_{01}))^{2\beta} + P_i(\set{|\eta(X)-0.5|<\epsilon})^{2\beta} \big]
            \end{aligned}
        \end{equation*}
        Denoting the term inside the summation as \(\tilde{H}\) we have, using Azuma-Mcdiarmid's inequality, with probability \(1-\delta\) 
        \begin{equation*}
            \begin{aligned}
                \n{(\hat{T}_{\hat{h}} - \hat{T}_{h^\ast})r}^2
                & \leq \frac{ 4 \n{r}^2 B_\ck^2 L_\ck }{\epsilon^{2\beta}} \big[ \sqrt{ \frac{\log \delta^{-1}}{N_{dr}}} + \E_{P_{XY} \sim \rho}[\tilde{H}] \big]
            \end{aligned}
        \end{equation*}
        Now, for a sequence \(\epsilon_j\) such that \(\epsilon_j^{2\beta} \sqrt{N_{dr}} \rightarrow \infty\) then the first term in the above equation converges in probability. Using the same arguments used at the end of theorem \ref{th:lt1} the second term converges as well.
    \end{proof}

    \subsection{Proof for theorem \ref{th:ltconsistency}} \label{proof:th:ltconsistency}

    \begin{proof}

        \begin{equation}
            \begin{aligned}
                \sE(\hat{r} \circ \hat{h}) - \sE(r_\sH \circ h^\ast) & = \n{A_{\hat{h}} \hat{r} - \gamma}_\rho^2 - \n{A_{h^\ast} r_{\sH_\ck} - \gamma}_{\rho}^2 \\
                    & = \n{A_{\hat{h}} \hat{r} - A_{h^\ast} r_{\sH_\ck}}_{\rho}^2 + 2 \ip{A_{\hat{h}} \hat{r} - A_{h^\ast} r_{\sH_\ck}, A_{h^\ast} r_{\sH_\ck} -\gamma}_\rho \\
                    & = \n{A_{\hat{h}} \hat{r} - A_{h^\ast} \hat{r} + A_{h^\ast} \hat{r} - A_{h^\ast} r_\sH}_\rho^2 + 2 \ip{A_{\hat{h}} \hat{r} - A_{h^\ast} r_{\sH_\ck}, A_{h^\ast} r_{\sH_\ck} -\gamma}_\rho \\
                    & \leq 2 (\n{A_{\hat{h}} \hat{r} - A_{h^\ast} \hat{r}}_{\rho}^2 + \n{A_{h^\ast} \hat{r} - A_{h^\ast} r_{\sH_\ck}}_{\rho}^2 + \ip{A_{\hat{h}} \hat{r} - A_{h^\ast} r_{\sH_\ck}, A_{h^\ast} r_{\sH_\ck} -\gamma}_\rho )\\
                    & = (I) + (II) + (III),
            \end{aligned}
        \end{equation}
        with the inequality coming from the fact that \(\n{\sum_{i=1}^N f_i}^2 \leq N \sum_{i=1}^N \n{f_i}^2\). We now provide universal consistency of \((I), (II)\) and \((III)\) using theorems \ref{th:defAT}, \ref{th:mincrit} and \ref{th:lt1}

    \textbf{Bound on (I)} We have from theorem \ref{th:lt1},

    \begin{equation}
        \begin{aligned}
            (I) & \leq \n{\hat{r}}_{\sH_{\bar{k}}}^2 \n{(A_{\hat{h}} - A_{h^\ast})}_{\sL(L^2)}^2
        \end{aligned}
    \end{equation}

    We have (similar to derivation of equation (33) in \citet{szabo2014two})
    \begin{equation}
        \begin{aligned}
            \label{eq:nrbound}
            \n{\hat{r}}_{\sH_\ck} & \leq \n{(A_h^\ast A_h + \xi_2)^{-1}}_{\mathcal{L}(\sH_\ck)} \n{\f{1}{N} \sum_{l = 1}^n \ck(\cdot,h_y(\phi_{p_l}))\gamma_{y,l}} \\
            & \leq \f{1}{\xi_2^2} C_\gamma^2 B_\ck,
        \end{aligned}
    \end{equation}
    If, in addition to the requirements of theorem \ref{th:lt1}, we have a sequence \(\xi_{2,j}\) such that \(\xi_{2,j} \rightarrow 0\) and \\ \(\xi_{2,j} \xi_{1,j} \epsilon_j^{2\beta} \sqrt{\frac{j}{\log j}} \rightarrow \infty\) then \((I) \rightarrow 0\) in probability
    

    \textbf{Bound on (III)}:
    
    \begin{equation}
        \begin{aligned}
            \label{eq:lthmtwotermsetup}
            (III) & = \ip{A_{\hat{h}} \hat{r} - A_{h^\ast} \hat{r}, A_{h^\ast} r_{\sH_\ck} -\gamma}_\rho + \ip{ A_{h^\ast} \hat{r} - A_{h^\ast} r_{\sH_\ck}, A_{h^\ast} r_{\sH_\ck} -\gamma}_\rho \\
            & = \ip{A_{\hat{h}} \hat{r} - A_{h^\ast} \hat{r}, A_{h^\ast} r_{\sH_\ck} -\gamma}_\rho \\
            &\leq \n{A_{\hat{h}} - A_{h^\ast}}_{\mathcal{L}(L^2)} \n{\hat{r}}_{\sH_\ck} \n{A_{h^\ast} r_{\sH_\ck} - \gamma}_{\sJ}
        \end{aligned}
    \end{equation}

    In addition, we can bound \(\n{A_{h^\ast} r_{\sH_\ck} - \gamma}_{\sJ}\) using the union bound and \ref{eq:nrbound} as 

    \begin{equation}
        \begin{aligned}
            \label{eq:Ahgammabound}
            \n{A_{h^\ast} r_{\sH_\ck} - \gamma}_{\sJ} & \leq \n{A_{h^\ast} r_{\sH_\ck}} + \n{\gamma} \\
            & \leq C_\gamma \bigg( 1 + \f{B_{\ck}^2}{\lambda_2} \bigg)
        \end{aligned}
    \end{equation}

    We now have,
    \begin{equation}
        (III) \leq \frac{C_\gamma^2 \sqrt{B_{\ck}}}{\xi_2} \bigg( 1 + \f{B_{\ck}^2}{\xi_2} \bigg) \n{A_{\hat{h}} - A_{h^\ast}}_{\mathcal{L}(L^2)} 
    \end{equation}

    Using equation \ref{eq:nrbound}, \ref{eq:Ahgammabound} and the sequence \(\xi_{2,j}\) constructed in term \((I)\), \((III) \rightarrow 0\).

    Finally, we have term \((II)\)
    \begin{equation*}
        \begin{aligned}
            \label{eq:lthmtwotermtwo}
            (II) & = \n{\sqrt{T_{h^\ast}}(\hat{r} - r_{\sH_\ck})}_{\sH_\ck}^2 \\
        \end{aligned}
    \end{equation*}
    Note that \(\hat{r}\) is trained with two stage sampled data which has been classified by \(\hat{h}\) and not by \(h^\ast\). To distinguish between the two we shall make a change to the notation: \(\hat{r}_{\hat{h}} = \hat{r}\) and we shall denote by \(\hat{r}_{h^\ast}\) as the empirical two stage regressor trained using the optimal marginal predictor. We have
    \begin{equation*}
        \begin{aligned}
            (II) &= \n{\sqrt{T_{h^\ast}}(\hat{r}_{\hat{h}} - r_{\sH_\ck})}_{\sH_\ck}^2 \\
                 &= \n{\sqrt{T_{h^\ast}} (\hat{r}_{\hat{h}} - \hat{r}_{h^\ast} + \hat{r}_{h^\ast} - r_{\sH})}_{\sH_\ck}^2 \\
                 &\leq 2\n{\sqrt{T_{h^\ast}} (\hat{r}_{\hat{h}} - \hat{r}_{h^\ast})}_{\sH_\ck}^2 + \n{ \sqrt{T_{h^\ast}}(\hat{r}_{h^\ast} - r_{\sH})}_{\sH_\ck}^2 \\
                 &= (IIa) + (IIb).
        \end{aligned}
    \end{equation*}
    Working with the first term, and using the operator identity \(T_1^{-1} - T_2^{-1} = T_1^{-1}(T_1 - T_2)T_2^{-1}\) we have 
    \begin{equation*}
        \begin{aligned}
            \hat{r}_{\hat{h}} - \hat{r}_{h^\ast} &= (\hat{T}_{\hat{h}} + \xi_2)^{-1} \hat{s}_{\hat{h}} - (\hat{T}_{h^\ast} + \xi_2)^{-1} \hat{s}_{h^\ast} \\
                                                 &= (\hat{T}_{\hat{h}} + \xi_2)^{-1} (\hat{s}_{\hat{h}} - \hat{s}_{h^\ast}) + ((\hat{T}_{\hat{h}} + \xi_2)^{-1} - (\hat{T}_{h^\ast} + \xi_2)^{-1}) \hat{s}_{h^\ast} \\
                                                 &= (\hat{T}_{\hat{h}} + \xi_2)^{-1} (\hat{s}_{\hat{h}} - \hat{s}_{h^\ast}) + (\hat{T}_{\hat{h}} + \xi_2)^{-1} (\hat{T}_{\hat{h}} - \hat{T}_{h^\ast}) (\hat{T}_{h^\ast} + \xi_2)^{-1} \hat{s}_{h^\ast} \\
                                                 &= (\hat{T}_{\hat{h}} + \xi_2)^{-1} (\hat{s}_{\hat{h}} - \hat{s}_{h^\ast}) + (\hat{T}_{\hat{h}} + \xi_2)^{-1} (\hat{T}_{\hat{h}} - \hat{T}_{h^\ast}) \hat{r}_{h^\ast}
        \end{aligned}
    \end{equation*}
    Using the above simplification in term (IIa)
    \begin{equation*}
        \begin{aligned}
            (IIa) &\leq 2 \n{\sqrt{T_{h^\ast}}(\hat{T}_{\hat{h}} + \xi_2)^{-1}}_{\sL(\sH_\ck)}^2 (\n{\hat{s}_{\hat{h}} - \hat{s}_{h^\ast}}^2 + \n{\hat{T}_{\hat{h}} - \hat{T}_{h^\ast}}_{\sL(\sH_\ck)}^2 \n{\hat{r}_{h^\ast}}_{\sH_\ck}^2 )
        \end{aligned}
    \end{equation*}
    The term \(\n{\sqrt{T_{h^\ast}}(\hat{T}_{\hat{h}} + \xi_2)^{-1}}_{\sL(\sH_\ck)}^2 \leq \frac{4}{\xi_2} \) with probability \(1 - \delta_2/3\) using the bounds in \citet{szabo2014two} (equation 32).
    Here we have, using the characteristics of the training data
    \begin{equation*}
        \begin{aligned}
            \n{\hat{s}_{\hat{h}} - \hat{s}_{h^\ast}}^2 &\leq \frac{1}{N_{dr}} \sum_i \n{ (\ck_{(\hat{h}(\phi_i)} - \ck_{h^\ast(\phi_i)}) \gamma_i}^2 \\
                                                       &\leq \frac{ C^2_\gamma }{ N_{dr} } \sum_i \n{ \ck_{(\hat{h}(\phi_i)} - \ck_{h^\ast(\phi_i)}}^2 \\
                                                       &\leq \frac{ C^2_\gamma L_\ck }{ N_{dr} }\sum_i \n{\hat{h}(\phi_i) - h^\ast(\phi_i)}^{2\beta} 
        \end{aligned}
    \end{equation*}
    Following arguments similar to those used in lemma \ref{lem:lt2}, when \(\epsilon_j\) \(\xi_{2,j}\) and are such that \(\epsilon_j \rightarrow 0, \xi_{2,j} \rightarrow 0\) and \(\epsilon_j^{2\beta} \xi_{2,j} \sqrt{j} \rightarrow \infty \) then \(\n{\sqrt{T_{h^\ast}}(\hat{T}_{\hat{h}} + \xi_2)^{-1}}_{\sL(\sH_\ck)}^2 \n{\hat{s}_{\hat{h}} - \hat{s}_{h^\ast}}^2 \rightarrow 0\) in probability for \(N_{dr} \geq j\). With the same sequence, using the properties of lemma \ref{lem:lt2} we have convergence of \( \n{\sqrt{T_{h^\ast}}(\hat{T}_{\hat{h}} + \xi_2)^{-1}}_{\sL(\sH_\ck)}^2 \n{\hat{T}_{\hat{h}} - \hat{T}_{h^\ast}}_{\sL(\sH_\ck)}^2 \).

    The convergence characteristics for term (IIb) are governed by theorem 2 in \citet{szabo2016learning} and we refer the reader to it for convergence. Note that if the convergence requirements for term (IIb) as stated in \citet{szabo2016learning} are satisfied then a sequence can be constructed for convergence of term (IIa). Therefore, if the convergence criteria for term (IIb) are met for universal consistency and the sequence \(\epsilon_j\) exists then \(\sE(\hat{r} \circ \hat{h}) \rightarrow \sE(r_\sH \circ h^\ast)\) and the entire system is consistent.

    \end{proof}

\end{document}